%% file: main.tex
\documentclass[11pt]{amsart}

\usepackage[margin=1in]{geometry}
\usepackage{amsmath,amssymb,amsfonts,amsbsy,amscd}
\usepackage{mathrsfs,latexsym,bm,extarrows}
\usepackage{tikz}
\usetikzlibrary{patterns}
\usepackage{graphicx}
\usepackage[caption=false]{subfig}
\usepackage{enumitem}
\usepackage{cases}
\usepackage{xcolor}
\usepackage[sort]{cite}
\usepackage{hyperref}
\hypersetup{hidelinks}

\newcommand{\rd}{\mathrm{d}}

\newcommand{\norm}[1]{ \| #1 \|}

\newcommand{\uvec}{\mathsf{u}}
\newcommand{\yvec}{\mathsf{y}}
\newcommand{\Amat}{\mathsf{A}}

\newcommand{\Imat}{\mathsf{I}}
\newcommand{\Vmat}{\mathsf{V}}
\newcommand{\Umat}{\mathsf{U}}

\newcommand{\calG}{\mathcal{G}}
\newcommand{\calM}{\mathcal{M}}
\newcommand{\sfR}{\mathsf{R}}

\newtheorem{remark}{Remark}
\newtheorem{lemma}{Lemma}
\newtheorem{proposition}{Proposition}

\newtheorem{theorem}{Theorem}

\newtheorem{corollary}{Corollary}

\DeclareMathOperator*{\argmin}{arg\,min}
\newcommand{\Col}{\text{Col}}
\newcommand{\calP}{\mathcal{P}}

\newcommand{\calR}{\mathcal{R}}
\newcommand{\calD}{\mathcal{D}}
\newcommand{\wass}{\mathcal{W}}
\newcommand{\opt}{\textrm{opt}}

\newcommand{\eq}{\infty}

\newcommand{\KL}{\text{KL}}
\newcommand{\ac}{\text{ac}}
\newcommand{\Row}{\text{Row}}
\newcommand{\bbR}{\mathbb{R}}

\newcommand*{\rom}[1]{\expandafter\@slowromancap\romannumeral #1@}

\newcommand\numberthis{\addtocounter{equation}{1}\tag{\theequation}}

\title[Stochastic Inverse Problem]{Stochastic Inverse Problem: Stability, Regularization and Wasserstein Gradient Flow}
\date{}

\author{Qin Li}
\address{Department of Mathematics, University of Wisconsin-Madison, Madison, WI, USA}
\email{qinli@math.wisc.edu}

\author{Maria Oprea}
\address{Institute for Science and Technology Austria}
\email{moprea@ista.ac.at}

\author{Li Wang}
\address{School of Mathematics, University of Minnesota Twin Cities, Minneapolis, MN, USA}
\email{liwang@umn.edu}

\author{Yunan Yang}
\address{Department of Mathematics, Cornell University, Ithaca, NY, USA}
\email{yunan.yang@cornell.edu}

\thanks{Q.~Li is partially supported by NSF grants DMS-2023239, DMS-2308440, and DMS-2608503. L.~Wang is partially supported by NSF grants DMS-1846854 and DMS-2513336 and by UMN DSI-SSG-4886888864. M.~Oprea and Y.~Yang acknowledge support from NSF through grant DMS-2409855, Office of Naval Research through grant N00014-24-1-2088, and Cornell PCCW Affinito-Stewart Grant. Y.~Yang is additionally supported by NSF grant DMS-2540324.}

\begin{document}

\begin{abstract}
Inverse problems in physical or biological sciences often involve recovering an unknown parameter that is random. The sought-after quantity is a probability distribution of the unknown parameter, that produces data that aligns with measurements. Consequently, these problems are naturally framed as stochastic inverse problems. In this paper, we explore three aspects of this problem: direct inversion, variational formulation with regularization, and optimization via gradient flows, drawing parallels with deterministic inverse problems. A key difference from the deterministic case is the space in which we operate. Here, we work within probability space rather than Euclidean or Sobolev spaces, making tools from measure transport theory necessary for the study. Our findings reveal that the choice of metric---both in the design of the loss function and in the optimization process---significantly impacts the stability and properties of the optimizer.
\end{abstract}

\maketitle


\section{Introduction}
Inverse problems focus on inferring parameters from data. Given the forward map $\mathcal{G}$ and the collected data $y$, which approximates the true data $y^\ast$, one seeks a parameter $u$ such that
\begin{equation}\label{eqn:inversion_det}
    \mathcal{G}(u) = y \quad\Longrightarrow\quad u = \mathcal{G}^{-1}(y)\,.
\end{equation}
When $\mathcal{G}$ is not invertible, $\mathcal{G}^{-1}$ should be interpreted as a pre-image. Practical problems introduce additional complexities. First, $\mathcal{G}^{-1}$ may not be uniquely defined, and the data $y = y^\ast + \delta$ may include measurement error $\delta$. To address these issues, one typically adopts a variational framework, seeking a solution to the following optimization problem:
\begin{equation}\label{eqn:optimization_det}
    \min_u L(u) = \|\mathcal{G}(u) - y\| + \mathsf{R}(u)\,.
\end{equation}
Here the norm in the first term and the choice of the regularization term $\mathsf{R}$ depends on prior knowledge about the properties of $u$ and $\mathcal{G}$~\cite{engl1996regularization}. Classical examples include using the total variation (TV) norm~\cite{rudin1992nonlinear} or $L^1$ norm~\cite{candes2006stable} for $\mathsf{R}$ to promote sparsity, and the $L^2$ norm (i.e., mean squared error) for the data fidelity term to account for measurement error.

The formulation in~\eqref{eqn:optimization_det} motivates the development of various solvers, with one of the most prominent being the gradient descent method~\cite{nocedal1999numerical,boyd2004convex}. The continuous-time limit of this method is given by:
\begin{equation}\label{eqn:gradient_det}
    \dot{u} = \frac{\rd}{\rd t} u = -\nabla_u L\,.
\end{equation}
The objective is that, in pseudo-time $t$, the parameter $u(t)$ evolves towards the point that minimizes~\eqref{eqn:optimization_det} with a proper initial guess $u(0)$.

The combination of~\eqref{eqn:inversion_det},~\eqref{eqn:optimization_det}, and~\eqref{eqn:gradient_det} raises several important questions, both qualitatively and quantitatively. Qualitatively, one may ask whether~\eqref{eqn:optimization_det} has a unique solution and whether it adequately approximates~\eqref{eqn:inversion_det}. Additionally, does the process described by~\eqref{eqn:gradient_det} converge? Quantitatively, how closely does the solution to~\eqref{eqn:optimization_det} approximate the solution to~\eqref{eqn:inversion_det}, and how fast does the gradient descent method in~\eqref{eqn:gradient_det} converge?

Many of these questions have been answered beautifully in specific contexts, driving significant research that underpins the foundations of Tikhonov regularization~\cite{golub1999tikhonov,engl1996regularization}, total variation denoising~\cite{rudin1992nonlinear}, and compressive sensing~\cite{candes2006stable}. Our aim is to lift all of these discussions on inverse problems, from the Euclidean space, to the space of probability distributions.

Lifting these problems up to the probability space is not only a mathematically interesting question, but also is backed by substantial practical demand. Over recent years, inverse problems associated with finding probability measures have gained increasing prominence. For example, in weather prediction, the goal is to infer the distribution of pressure and temperature changes~\cite{gneiting2014probabilistic}; in plasma simulation, one aims to infer the distribution of plasma particles using macroscopic measurements~\cite{Ferron_1998,caflisch2021adjoint}; in experimental design, the objective is to determine the optimal distribution of tracers or detectors to achieve the best measurements~\cite{huan2024optimalexperimentaldesignformulations,jin2024optimaldesignlinearmodels,yu2018scalable}; and in optical communication, the task is to recover the distribution of the optical environment~\cite{bracchini2004spatial,Korotkova:11,Borcea2015}. Other problems include those arising in aerodynamics~\cite{del2022stochastic}, biology~\cite{davidian2003nonlinear,daun2008ensemble,tang2023ensemble}, and cryo-EM~\cite{giraldo2021bayesian,tang2023ensemble}. In all these problems, the sought-after quantity is a probability distribution, density, or measure that matches the given data. Consequently, inverse problems in this stochastic setting are naturally formulated as the inversion for a probability distribution, giving rise to the so-called stochastic inverse problem~\cite{breidt2011measure,butler2012computational,butler2013numerical,butler2014measure,marcy2022stochastic,li2023differential,white2024building}.

We are now tasked with translating the~\eqref{eqn:inversion_det}-\eqref{eqn:optimization_det}-\eqref{eqn:gradient_det} framework into the stochastic setting. The same three problems will be investigated in this new context. Throughout this paper, we assume that the push-forward map $\mathcal{G}$ is known~\cite{Villani}, meaning that for any given $u$, we can efficiently evaluate $\mathcal{G}(u)$. Although it may be computationally expensive, we also assume that $\nabla_u\mathcal{G}$ can be evaluated. Additionally, we assume that the measured data distribution $\rho^\delta_y$ is within a $\delta$-distance (the specific definition of this distance will be clarified later in the appropriate context) from the ground truth data distribution $\rho_y^\ast = \mathcal{G}\#\rho_u^\ast$, meaning $\rho_y^\ast$ is obtained by push-forwarding $\rho_u^\ast$ through $\calG$, where $\rho_u^* $ is the true parameter distribution. Our objective is to
\vspace{0.1in}
\begin{center}
    \emph{design a formulation and a solver to find $\rho_u$  that approximates $\rho_u^\ast$ from data $\rho^\delta_y$.}
\end{center}
\vspace{0.1in}

Similar to the deterministic case, we consider the following three problems:
\begin{itemize}
    \item \textbf{Problem I: Direct Inversion.} This involves solving
    \begin{equation}\label{eqn:inversion_sto}
    \rho^\delta_u = \mathcal{G}^{-1}\#\rho^\delta_y\,.
    \end{equation}
    We need to understand the meaning of $\mathcal{G}^{-1}$ when $\mathcal{G}$ is not invertible. Additionally, we will assess the error between $\rho^\delta_u$, the reconstructed distribution, and $\rho_u^\ast$, the ground truth, when $\rho^\delta_y$ is within a $\delta$-ball of $\rho^\ast_y$ for a given distance/divergence. This problem mirrors~\eqref{eqn:inversion_det}.

    \item \textbf{Problem II: Variational Formulation.} The objective here is to define an appropriate functional $E[\rho_u;\rho^\delta_y]$ and solve the optimization problem
    \begin{equation}\label{eqn:optimization_sto}
    \rho^\delta_u = \argmin_{\rho_u \in \mathcal{P}} E[\rho_u; \rho^\delta_y]\,,
    \end{equation}
    where $\mathcal{P}$ represents the space of probability distributions. This variational approach reformulates Problem I. The goal remains to approximate $\rho_u^\ast$ by $\rho_u^\delta$, given that $\rho_y^\delta$ is a $\delta$-perturbation of $\rho_y^\ast$. A well-defined $E$, combined with a structured regularization term, can further ensure that $\rho^\delta_u$ closely approximates $\rho_u^\ast$. This is analogous to~\eqref{eqn:optimization_det}.

    \item \textbf{Problem III: Gradient Flow Structure.} Here, the focus is on analyzing the gradient-based solver
    \begin{equation}\label{eqn:gradient_sto}
    \partial_t \rho_u = -\nabla E\,,
    \end{equation}
    and its performance on the space $\mathcal{P}$, the collection of all probabilities. It is important to note that the gradient of the energy functional, $\nabla E$, is metric-dependent. Different choices of metrics and properties of $E$ can significantly impact convergence. This problem corresponds to~\eqref{eqn:gradient_det}.
\end{itemize}
In summary, our aim is to extend key formulations from the deterministic inverse problem, \eqref{eqn:inversion_det}-\eqref{eqn:optimization_det}-\eqref{eqn:gradient_det}, to their counterparts in the space of probability measures, \eqref{eqn:inversion_sto}-\eqref{eqn:optimization_sto}-\eqref{eqn:gradient_sto}, as illustrated in~Figure~\ref{fig:det_sto_diagram}.

\begin{figure}[h!]
\centering
\input{Figures/diagramSIP}
\caption{A diagram showing the relations between deterministic inverse problem~\eqref{eqn:inversion_det} and the stochastic inverse problem~\eqref{eqn:inversion_sto} formulated based on the push-forward map.}\label{fig:det_sto_diagram}
\end{figure}
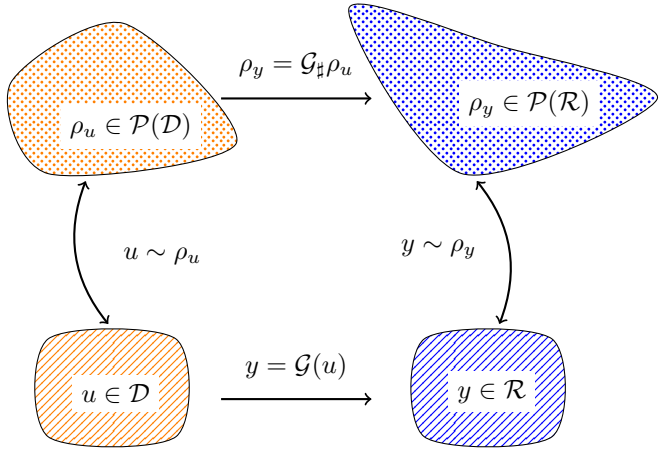

It is impossible to address all the above questions in their most general settings in one paper. Here, we will tackle some fundamental ones and establish connections with their deterministic counterparts. The key findings of our study are:

\begin{itemize}
    \item[1.] The stability of direct inversion is highly dependent on the metric used to measure the reconstruction, both in the invertible case (Theorem~\ref{thm:inv_stability}) and the under-determined case (Theorem~\ref{thm:under-stability}). Notably, the Wasserstein distance (e.g., $\wass_2$) is more sensitive to data perturbations than $f$-divergences.

    \item[2.] In the variational formulation, the choice of the regularizer and its relationship with the main objective function play a crucial role in the optimizer's behavior. We explore both entropy-entropy and $\wass_2$-$\wass_2$ pairings, observing a \textit{strong similarity to the classical Tikhonov regularization}. The optimal value of the regularization coefficient depends on the size of $\delta$, and these details are outlined in Theorem~\ref{reg-KL} and Theorem~\ref{reg-W}.

    \item[3.] In the gradient flow formulation, we find that the form of the objective function leads to distinct equilibrium solutions. Interestingly, as demonstrated in Theorem~\ref{thm:equilibrium_form}, the recovery corresponds to a \textit{conditional distribution} in the case of $f$-divergence and a \textit{marginal distribution} in the case of $\wass_2$, under some assumptions.
\end{itemize}

In the subsequent sections, Section~\ref{sec:problemI}, Section~\ref{sec:problemII}, and Section~\ref{sec:problemIII}, we examine Problems I, II, and III as posed above, respectively. Throughout the paper, we denote the map
\begin{equation}\label{eqn:G_def}
\mathcal{G}: \mathcal{D} \subset \mathbb{R}^m \to \mathcal{R} \subset \mathbb{R}^n
\end{equation}
taking the domain $\mathcal{D}$ to the range $\mathcal{R}$. For the sake of precise statements, we occasionally consider $\mathcal{G}$ as a linear map, with $\mathcal{G} = \Amat \in \mathbb{R}^{n \times m}$ representing a matrix. The matrix $\Amat$ may vary in size depending on whether the problem is overdetermined or underdetermined, but it is always assumed to be full-rank, meaning that the number of non-zero singular values equals $\min\{m, n\}$. We denote the smallest singular value as $\sigma_{\min}(\Amat)$. Additionally, we use $\Amat^\dagger$ to denote the Moore--Penrose inverse of $\Amat$, given by
\begin{equation}\label{eqn:A_penrose}
\Amat^\dagger = \begin{cases}
(\Amat^\top \Amat)^{-1} \Amat^\top\,,\quad &\text{when } n > m \text{ and the system is overdetermined,} \\
\Amat^\top(\Amat \Amat^\top)^{-1}\,,\quad &\text{when } n < m \text{ and the system is underdetermined.}
\end{cases}
\end{equation}

Moreover, $\mathcal{P}(\Omega)$ denotes the collection of probability measures whose support lies within $\Omega$. When the subscript ``ac" is used, we focus exclusively on probability measures that are absolutely continuous with respect to the Lebesgue measure, meaning they have probability density functions. When the subscript $n$ appears, we consider the subset of $\mathcal{P}$ whose $n$-th order moment is finite. For example, $\mathcal{P}_2$ includes all probability measures with bounded second-order moments.

Two classes of discrepancy measurement will be employed: the Wasserstein metric and the $f$-divergence. Specifically, the $p$-Wasserstein distance between two probability measures is defined as:
\begin{equation}\label{eq:wass_p}
\wass_p(\mu, \nu) = \left(\min_{\gamma \in \Gamma(\mu, \nu)} \int \|x - y\|^p \,\mathrm{d}\gamma\right)^{1/p}\,, \quad p \geq 1\,,
\end{equation}
where $\Gamma$ represents the set of all couplings between the two measures. By definition, $\wass_p$ is only applicable in the space $\mathcal{P}_p$. The general $f$-divergence is defined as:
\begin{equation}\label{eq:f-divergence}
D_f(\mu \,\|\, \nu) = \int f\left(\frac{\mathrm{d} \mu}{\mathrm{d} \nu}\right) \,\mathrm{d} \nu\,,
\end{equation}
for a convex function $f$. According to this definition, $\mu$ must be absolutely continuous with respect to $\nu$, i.e., $\mu \ll \nu$, for the $f$-divergence to be well-defined. One classical example in this category is the KL divergence where
\[
f(x) = x\ln(x)\,,\quad \KL(\mu \,\|\, \nu) = \int\ln\frac{\rd\mu}{\rd\nu}\rd\mu\,.
\]

\section{Problem I: direct inversion, wellposedness and stability}\label{sec:problemI}
This section is dedicated to Problem I, direct inversion. More specifically, we study~\eqref{eqn:inversion_sto}, and the problems associated with its formulation: the definition and stability. To frame the problem in the context, we first review our knowledge in the deterministic setting, before lifting it up to our setting.

\subsection{Direct inversion in the deterministic setting}
We now examine~\eqref{eqn:inversion_det} in the standard Euclidean space equipped with the $L^2$ norm. It is classical knowledge that if $\mathcal{G}$ is invertible, and $y\in\mathcal{R}$, then
\[
u=\mathcal{G}^{-1}(y)
\]
is a well-defined quantity. Moreover, denote the control of the measurement error $\|y-y^\ast\|\leq \delta$. If $\mathcal{G}^{-1}$ is $\beta$-H\"older continuous, for some $\beta \in (0,1]$, that is
\begin{equation}\label{eqn:holder_cont}
\|\mathcal{G}^{-1}(y_1)-\mathcal{G}^{-1}(y_2)\|\leq C \norm{y_1 - y_2}^\beta, \ \forall y_1, y_2 \in \mathcal{D}
\end{equation}
for some $C$, we quickly have the stability
\begin{equation}\label{eqn:stability_det_nonlinear}
\|u-u^\ast\|\leq C\delta^\beta\,.
\end{equation}

The problem becomes interesting when $\mathcal{G}$ is not invertible. In this case, $\mathcal{G}^{-1}$ should be understood as the pre-image, and the solution is thus not unique. The stability highly depends on the specifics of $\mathcal{G}$, and if $\mathcal{G}$ is linear, the problem can be analyzed in a more generic form.

Let $\Amat \in \mathbb{R}^{n \times m}$ be an underdetermined matrix of full rank, i.e., $m > n$. We would like to invert the operation $\Amat\uvec = \yvec$. The solution is non-unique, so we can only analyze stability in terms of the distance between the solution sets. To this end, we view $\Amat^{-1}$ as the pre-image operator. For every $\yvec\in\mathbb{R}^{n}$, define $S_\yvec :=  \{\uvec | \Amat\uvec = \yvec\}$. Clearly for linear systems,
\begin{equation}\label{eqn:split_set}
S_\yvec = \{ \Amat^\dagger \yvec  +  \uvec_0 | \Amat\uvec_0 = 0\}=\underbrace{\Amat^\dagger\yvec}_{\in\text{Row}(\Amat)}+\,\,\mathcal{N}(\Amat)\,,
\end{equation}
where $\Amat^\dagger$ is defined in~\eqref{eqn:A_penrose} and $\mathcal{N}(\Amat)$ denotes the null space of $\Amat$. Note that $\mathcal{N}(\Amat)^\perp = \text{Row}(\Amat)$, and the decomposition above is composed of $S_\yvec$'s projection on two subspaces and the orthogonal decomposition of each element is unique (see Figure \ref{fig:decomp}). Since $\Amat^\dagger\yvec$ is the projection of the set $S_\yvec$ onto $\mathcal{N}(\Amat)^\perp$, it can also be interpreted as:
\[
\Amat^\dagger \yvec = \argmin_{\uvec} \|\uvec\|^2_2 \text{\quad  subject to \quad } \Amat\uvec = \yvec\,.
\]
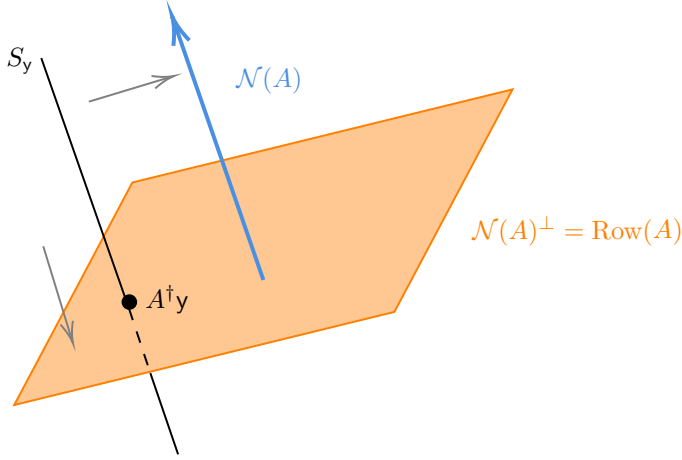
\begin{figure}
    \centering
    \input{Figures/determinisitc}
    \caption{Orthogonal decomposition of the domain of $\Amat$.}
    \label{fig:decomp}
\end{figure}

We now define the distance between two solution sets as:
\begin{equation}\label{eqn:def_distance_set_det}
d (S_\yvec, S_{\yvec'}) = \inf_{\uvec \in S_\yvec, \uvec' \in S_{\yvec'}} \|\uvec - \uvec'\|\,,
\end{equation}
where we adopt the standard Euclidean distance. The decomposition~\eqref{eqn:split_set} allows us to easily compute this distance:
\begin{equation}\label{eqn:stability_det_linear}
    \begin{aligned}
        d (S_\yvec, S_{\yvec'}) &= \inf_{\uvec \in S_\yvec, \uvec' \in S_{\yvec'}} \|\uvec - \uvec'\| = \|\Amat^\dagger \yvec  - \Amat^\dagger\yvec'\|+\inf_{\uvec_{1,2} \in \mathcal{N}(\Amat)}\|  \uvec_1-\uvec_2\|\\
        &\leq \|\Amat^\dagger\|\|\yvec-\yvec'\|\leq \frac{\|\yvec-\yvec'\|}{\sigma_{\min}(\Amat)}\,.
    \end{aligned}
\end{equation}
Here we have used $\inf_{\uvec_{1,2} \in \mathcal{N}(\Amat)}\|  \uvec_1-\uvec_2\|=0$ and that $\|\Amat^\dagger\| = 1/\sigma_{\min}(\Amat)$, where $\sigma_{\min}$ is the smallest singular value of $\Amat$.

\begin{remark}
Considering that both $S_\yvec$ and $S_\yvec'$ are linear spaces and are not overlapping, the largest distance between the two sets is $\infty$. This can be achieved by setting $\uvec = \Amat^\dagger \yvec + \uvec_0$ and $\uvec' = \Amat^\dagger \yvec' + n\, \uvec_0$ with $n\to \infty$.
\end{remark}

\subsection{Direct inversion in the stochastic setting}
To lift the discussion to the stochastic setting, we are looking for the solution to~\eqref{eqn:inversion_sto}. Similar to the deterministic setting, we would like to understand how the changes in $\rho_y$ propagate to $\rho_u$, both when $\mathcal{G}$ is invertible and when it is under-determined. These studies will lead to analogue results of~\eqref{eqn:stability_det_nonlinear} and \eqref{eqn:stability_det_linear}.

\subsubsection{When $\calG$ is invertible}
When $\calG$ is a bijection and $\calG^{-1}$ exists and is unique, we consider the data distribution $\rho_y^* \in \mathcal{P}(\mathcal{R})$. We can obtain
\[
\rho_u^* = \calG^{-1}\#\rho_y^*\,,
\]
as one solution to the stochastic inverse problem~\eqref{eqn:inversion_sto}. This solution is unique. Suppose $\rho_{u,2}$ is another solution so that $\calG\#\rho_{u,2}=\rho_y^*$, then
\[
\rho^*_u = \calG^{-1}\# \rho^*_y =\calG^{-1}\# \left( \calG\# \rho_{u,2}\right)  = \left(\calG^{-1}\circ\calG\right)\# \rho_{u,2}= \rho_{u,2}\,.
\]

To evaluate the stability, the problem becomes more convoluted than that in the deterministic setting. The metric to quantify error (a distance between two probability measures) needs to be pre-determined. In this infinite dimensional setting, different metrics can lead to significantly different stability.

\begin{theorem}\label{thm:inv_stability}
Consider the push-forward of a map $\calG:\calD\longrightarrow\mathcal{R}$~\eqref{eqn:G_def} and assume $\calG$ is invertible, with its inverse $\calG^{-1}$ being $\beta$-continuous for a constant $C_{\calG^{-1}}$; see~\eqref{eqn:holder_cont}. Then given two data distributions $\rho^\ast_y\in \mathcal{P}(\calR)$ and its perturbation $\rho_y^\delta \in\mathcal{P}(\calR)$, we define $\rho^\ast_u = \calG^{-1}\# \rho^\ast_y$ and $\rho_u^\delta = \calG^{-1}\# \rho_y^\delta$ respectively and have the following stabilities:
    \begin{itemize}
        \item[1)] $\beta$-continuous in the Wasserstein sense:
\begin{align}\label{eq:Wass_stab}
            \wass_p(\rho^\ast_u, \rho_u^\delta) \leq C_{\calG^{-1}} \wass_p (\rho^\ast_y,  \rho_y^\delta)^\beta\,,
        \end{align}
        \item[2)] Lipschitz continuous in the  $f$-divergence sense:
          \begin{align*}
            D_f(\rho^\ast_u || \rho_u^\delta) = D_f  (\rho^\ast_y|| \rho_y^\delta)\,.
        \end{align*}
    \end{itemize}
\end{theorem}
\begin{proof}
For the $p$-Wasserstein case, the result directly follows from~\cite[Theorem 3.2]{ernst2022wasserstein}. If $D_f$ is the $f$-divergence, then by the data processing inequality \cite{beaudry2012intuitive}:
\begin{equation}\label{eq:dpi}
D_f(\rho^\ast_y||\rho_y^\delta) = D_f(\calG\# \rho^\ast_u||\calG\# \rho_u^\delta) \leq
D_f(\rho^\ast_{u}|| \rho_{u}^\delta)\,.
\end{equation}
On the other hand, we have
\begin{align}\label{dpi2}
  D_f(\rho^\ast_{u}|| \rho_{u}^\delta) =  D_f(\calG^{-1}\# \rho_y||\calG^{-1}\# \rho_y^\delta)
   \leq
D_f(\rho_{y}|| \rho_{y}^\delta)\,.
\end{align}
Combing \eqref{eq:dpi} and \eqref{dpi2} leads to the result.
\end{proof}

Though straightforward in computation, this result is nevertheless alarming. The statement of the theorem suggests that when the perturbation is measured in $\wass_p$, we ``see" the continuity effect of the map $\calG^{-1}$, but such sensitivity is lost if $f$-divergence is used. A direct corollary derived from this is that when $\calG=\Amat$ is linear, $\calG^{-1}$ is Lipschitz continuous with index $\beta=1$ and the constant $C_{\calG^{-1}}=\frac{1}{\sigma_{\min}(\Amat)}$. On the contrary, $f$-divergence returns $1$-Lipschitz continuity in the reconstruction of $\rho_u$ even if $\calG$ is severely ill-conditioned.

\subsubsection{Under-determined case}
We discuss the situation when $\calG$ is not bijective in this subsection. Similar to the deterministic setting, when $\calG^{-1}$ cannot be uniquely defined on $\calR$, it should be understood as the pre-image, and the properties of the pre-image depend on the specific situation. We confine ourselves to the case where $\calG=\Amat$ is a linear map. As in the deterministic setting, we need to define the solution set for every given $\rho_y\in\mathcal{P}(\calR)$, and the distance between sets, as was done in~\eqref{eqn:def_distance_set_det}. In the current context, the solution set is simply:
\begin{equation}\label{eq:set_S}
S_{\rho_y} = \{\rho_u \in \mathcal{P}(\mathbb{R}^m) \ | \Amat\# \rho_u = \rho_y\}\,,
\end{equation}
and the distance between two sets $S_{\rho^1_y}$ and $S_{\rho^2_y}$ are, in the case of $f$-divergence:
\begin{equation}\label{eqn:def_dis_sto_f_divergence_inf}
d^f ( S_{\rho_y^1},  S_{\rho_y^2}) = \inf_{\substack{ \{\mu : \Amat\# \mu = \rho_y^1 \} \\ \{ \nu :\Amat\#\nu = \rho_y^2\} } } D_f (\mu || \nu)\,,
\end{equation}
and in the case of $\wass_2$:
\begin{equation}\label{eqn:def_dis_sto_wass_inf}
d^{\wass_2} ( S_{\rho_y^1},  S_{\rho_y^2}) = \inf_{\substack{ \{\mu : \Amat\# \mu = \rho_y^1 \} \\ \{ \nu :\Amat\#\nu = \rho_y^2\} } } \wass_2 (\mu , \nu)\,.
\end{equation}

As was suggested by Theorem~\ref{thm:inv_stability}, the sensitivity to the perturbation in $\rho_y$ heavily depends on the metric we use to evaluate the distances between measures. Indeed, we characterize the differences in Theorem~\ref{thm:under-stability} below. In its proof, we use the measure disintegration theorem~\cite[Thm.~5.3.1]{ambrosio2005gradient}. Here, we state a simplified version.
\begin{theorem}[Measure disintegration~\cite{disintegration}]\label{thm:disintegration}
    Let $\mu \in \mathcal{P}(Y)$, and consider $P:Y \to X$ a measurable function between the Radon spaces $Y$ and $X$. Define $\nu := P \# \mu$. Then there exists a $\nu$ a.e.~uniquely determined family of measures $\{\mu_x \}_{x\in X}\subset \mathcal{P}(Y)$ such that
    \begin{itemize}
        \item The map $x \mapsto \mu_x(\Omega)$ is Borel measurable for all  Borel sets $\Omega$.
        \item For $\nu$ a.e.~$x$, $\mu_x (Y \setminus P^{-1}(x)) = 0$.
        \item For every Borel measurable function $f:Y \to [0, \infty)$,
        \begin{equation}\label{eq:disintegration}
            \int_Y f(y) \rd\mu(y) = \int_X \int_{P^{-1}(x)} f(y) \rd \mu_x(y) \rd \nu(x)\,.
        \end{equation}
    \end{itemize}
\end{theorem}

\begin{theorem}\label{thm:under-stability}
Consider a matrix $\Amat \in \bbR^{n\times m}$ with $m > n$, and $\rho_y^1, \rho_y^2 \in \calP(\calR)$. Define $S_{\rho_y^1}, S_{\rho_y^2}\subset \calP(\calD)$ the two solution sets corresponding to data distributions $\rho_y^1, \rho_y^2$, respectively as in~\eqref{eq:set_S}. Then
\begin{itemize}
\item[1)] Lipschitz continuous in the Wasserstein sense:
\begin{equation}\label{eq:Wasserstein_stability_inf}
d^{\wass_2} ( S_{\rho_y^1},  S_{\rho_y^2}) =  \wass_2(\Amat^\dagger \#\rho_y^1, \Amat^\dagger \#\rho_y^2) \leq \left(\sigma_{\min}(\Amat)\right)^{-1}\wass_2( \rho_y^1,  \rho_y^2)  \,,
\end{equation}
\item[2)] Lipschitz continuous in the $f$-divergence sense:
\begin{eqnarray}\label{eqn:f_divergence_stability}
      d^f ( S_{\rho_y^1},  S_{\rho_y^2}) = D_f(\rho_y^1 || \rho_y^2)\,.\label{eq:f_divergence_stability_inf}
    \end{eqnarray}
\end{itemize}
\end{theorem}

This result is a one-to-one correspondence to Theorem~\ref{thm:inv_stability} in the setting where $\calG^{-1}$ is non-unique. Like before, the Wasserstein distance is sensitive to the behavior of $\calG$ while the $f$-divergence is blind to the conditioning of this map. However, the proof is much more convoluted.

\begin{proof}[Proof of~\eqref{eq:Wasserstein_stability_inf}] We first expand the definition~\eqref{eqn:def_dis_sto_wass_inf}. To do so, we adopt the orthogonal decomposition \eqref{eqn:split_set}. For all $u\in\bbR^m$:
\begin{equation}\label{eqn:orth_decomp}
u = P^R(u)+P^\perp(u)=u_2+u_1:=\Amat^\dagger \Amat u + (\Imat-\Amat^\dagger \Amat) u\,,
\end{equation}
where $P^R(u)$ projects $u$ onto $\text{Row}(\Amat)$ and $P^\perp(u)$ projects $u$ to $\mathcal{N}(\Amat)$. Furthermore, define
\[
P = P^R \otimes P^R\quad\text{with}\quad P(u, v) = (u_2, v_2)\,.
\]
We have the pre-image of $P^{-1}$, for $(u_2, v_2)\in Row(\Amat)\times Row(\Amat)$:
\[
P^{-1}(u_2, v_2) = \{ (u_2 + u_1, v_2 + v_1) | \forall\, u_1, v_1 \in \mathcal{N}(\Amat) \}\,.
\]
This separation allows us to control the 2-Wasserstein metric~\eqref{eq:wass_p}:
\begin{align*}
    \wass^2_2 (\mu , \nu)  &= \inf_{\gamma \in \Gamma(\mu, \nu)} \int_{\mathbb{R}^m \times \mathbb{R}^m } \|u - v\|^2 \rd\gamma(u, v) \\
    & = \inf_{\gamma  \in \Gamma(\mu, \nu)} \left( \int_{\mathbb{R}^m \times \mathbb{R}^m } \|u_1  - v_1\|^2 \rd \gamma(u, v)   +   \int_{\mathbb{R}^m \times \mathbb{R}^m }  \|u_2 - v_2\|^2 \rd \gamma(u, v) \right)\\
    &\geq  \inf_{\gamma \in \Gamma(\mu, \nu)} \int_{\mathbb{R}^m \times \mathbb{R}^m } \|u_2  - v_2\|^2 \rd \gamma(u, v)\\
     &=  \inf_{\gamma \in \Gamma(\mu, \nu)}\int\limits_{ \Row(\Amat)^2 } \int\limits_{\mathcal{N}(\Amat)^2} \|u_2 - v_2\|^2 \rd\gamma_{u_2, v_2}(u_1, v_1) \rd\left(P\#\gamma\right) (u_2, v_2)\\
    &= \inf_{\gamma \in \Gamma(\mu, \nu)}\int\limits_{\Row(\Amat)^2} \|u_2 - v_2\|^2 \bigg \{\int\limits_{\mathcal{N}(\Amat)^2} \rd\gamma_{u_2, v_2}(u_1, v_1) \bigg\} \rd\left(P\#\gamma\right)(u_2, v_2)\\
    &=  \inf_{\gamma \in \Gamma(\mu, \nu)}\int\limits_{\Row(\Amat)^2}\|u_2 - v_2\|^2  \rd \left(P\#\gamma\right)(u_2, v_2)\,.\numberthis\label{eq:wasserstein_0}
\end{align*}
where we applied the Measure Disintegration Theorem~\ref{thm:disintegration} on the coupling $\gamma$ with $f(u, v) = \|u_2 - v_2\|^2$ and deployed Equation~\eqref{eq:disintegration}. Noticing that $P\#\gamma$ is a measure on $\Row(\Amat)^2$, for any Borel measurable set $\Omega \subset \Row(\Amat)$, we have
\begin{align*}
 (P\#\gamma)( \Omega\times \Row(\Amat)) &= \gamma(P^{-1}(\Omega\times \Row(\Amat))) \\ &
 = \gamma ((P^R)^{-1}(\Omega) \times \mathbb{R}^m ) = \mu ((P^R)^{-1}(\Omega)) = (P^R\#\mu) (\Omega)\,,
\end{align*}
and similarly $(P\#\gamma)( \Row(\Amat) \times \Omega) = (P^R \# \nu)(\Omega)$. Hence, $P\# \gamma \in \Gamma(P^R \# \mu, P^R \# \nu)$ and~\eqref{eq:wasserstein_0} can be further simplified to
\begin{equation}\label{eq:wasserstein_1}
\begin{aligned}
     \wass_2^2(\mu, \nu) &\geq \inf_{\pi \in \Gamma(P^R \# \mu, P^R \# \nu)}\int_{\Row(A)^2} \|u_2 - v_2\|^2 \rd\pi(u_2, v_2)\\
     &= \wass_2^2(P^R \# \mu, P^R \# \nu)\,.
     \end{aligned}
 \end{equation}
Recall the requirement that $\Amat\#\mu = \rho_y^1$ and $\Amat\#\nu = \rho_y^2$. Then  $\forall \varphi:\Row(\Amat) \rightarrow \mathbb{R}$,  we have
\begin{align*}
    \int_{\Row(\Amat)} \varphi(u_2) \rd \left(\Amat^\dagger\# \rho_y^1\right)(u_2) &= \int_{\mathbb{R}^n} \varphi(\Amat^\dagger y) \rd\rho_y^1(y)  = \int_{\mathbb{R}^n} \varphi(\Amat^\dagger y) \rd (\Amat\#\mu)(y) \\
    &=  \int_{\mathbb{R}^m} \varphi(\Amat^\dagger \Amat u) \rd\mu(u)=  \int_{\mathbb{R}^m} \varphi\circ P^R(u) \rd\mu(u) \\
    &= \int_{\text{Row}(\Amat)} \varphi(u_2) \rd(P^R \# \mu)(u_2) \,,
\end{align*}
meaning that $ P^R \# \mu = \Amat^\dagger  \# \rho_y^1 $. A similar argument yields $P^R \# \nu = \Amat^\dagger \# \rho_y^2$. Therefore,~\eqref{eq:wasserstein_1} becomes
$$
\wass_2(\mu, \nu) \geq \wass_2(\Amat^\dagger \# \rho_y^1, \Amat^\dagger \# \rho_y^2), \ \forall\, \mu,\nu\, \text{ satisfying }\Amat\#\mu = \rho_y^1, \ \Amat\#\nu = \rho_y^2\,.
$$
Remembering that $\Amat^\dagger \#\rho_y^i \in S_{\rho_y^i}$ for $i = 1, 2$, we obtain
$$ \wass_2(\Amat^\dagger \#\rho_y^1, \Amat^\dagger \#\rho_y^2)\leq d_{\inf}^{\wass_2} (S_{\rho_y^1}, S_{\rho_y^2}) \leq  \wass_2(\Amat^\dagger \#\rho_y^1, \Amat^\dagger \#\rho_y^2)\,,$$
which implies Equation~\eqref{eq:Wasserstein_stability_inf}. The inequality in \eqref{eq:Wasserstein_stability_inf} follows from~\cite[Theorem 3.2]{ernst2022wasserstein}.
\end{proof}

\begin{proof}[Proof of~\eqref{eqn:f_divergence_stability}] We first note that if $\rho_y^1$ is not absolutely continuous with respect to $\rho_y^2$, both sides of~\eqref{eq:f_divergence_stability_inf} are infinite, and the result naturally holds. Therefore, we will assume that $D_f(\rho_y^1 || \rho_y^2) < \infty$ hereafter. Based on the data processing inequality:
\[
D_f(\rho_y^1 || \rho_y^2) = D_f(\Amat\# \mu ||\Amat\# \nu) \leq D_f(\mu || \nu), \ \forall \nu \in S_{\rho_y^2},\,\, \forall  \mu \in S_{\rho_y^1}\,.
\]
Hence, we obtain a lower bound for the infimum:
\begin{equation}\label{eqn:dp_forward}
D_f(\rho_y^1 \|\rho_y^2) \leq d_{\inf}^f ( S_{\rho_y^1},  S_{\rho_y^2}) = \inf_{\substack{ \{\mu : \Amat\# \mu = \rho_y^1 \} \\ \{ \nu :\Amat\#\nu = \rho_y^2\} } } D_f (\mu || \nu) \,.
\end{equation}
Let $\mathsf{B}$ be any inverse map that achieves:
$$\mathsf{B}:\mathcal{P}(\calR) \rightarrow \mathcal{P}(\calD)\,, \,\,\mathsf{B}(\rho_y) = \rho_u  \text{ such that }\Amat\# \rho_u = \rho_y\,.$$
One such example is to set $\mathsf{B}=\Amat^\dagger\#$. Let $\rho_u^1 = \mathsf{B}(\rho_y^1)$ and $ \rho_u^2 = \mathsf{B}(\rho_y^2)$. Define
\[
k(\rd x, y):= \mathsf{B}(\delta_y)(\rd x), \ \forall y \in \calR.
\]
Then considering $\mathsf{B}(\lambda_1\rho_1 + \lambda_2\rho_2) = \lambda_1 \mathsf{B}(\rho_1) + \lambda_2 \mathsf{B}(\rho_2)$ for all $\lambda_1,\lambda_2 \geq 0$ satisfying $\lambda_1 + \lambda_2 = 1$, we have:
\[
\mathsf{B}(\rho)(\Omega) = \int_\Omega \int_\calR k(\rd x, y) \rd \rho(y)\,,
\]
meaning $\mathsf{B}(\rho)$ represents a Markov transition over $\rho\in \calP(\calR)$. Thus, according to the data processing inequality again on $\mathsf{B}$:
\begin{equation*}
D_f(\rho_u^1 || \rho_u^2 ) =  D_f(\mathsf{B}(\rho_y^1) || \mathsf{B}( \rho_y^2) )\leq D_f(\rho_y^1 || \rho_y^2)\,.
\end{equation*}
Combining with~\eqref{eqn:dp_forward}, we arrive at~\eqref{eq:f_divergence_stability_inf}.
\end{proof}

\section{Problem II: variational formulation}\label{sec:problemII}
This section is dedicated to Problem II: the variational formulation, presented in the form of~\eqref{eqn:optimization_sto}. Data $\rho_y^\ast$ (or its perturbation $\rho^\delta_y$) is given. The clean data distribution $\rho_y^\ast$ is known to be produced by a push-forward map on a to-be-reconstructed $\rho^*_u$. An optimization formulation is a natural candidate to use for finding this $\rho^*_u$. When the direct inversion is either unavailable explicitly or ill-conditioned, this optimization formulation, in comparison to direct inversion, provides more flexibility for us to numerically handle the conditioning through the design of the objective functional.

In this section, we analyze two designs of the objective functional. In the first formulation, the objective is the most straightforward way of measuring the distance between the simulated data and the given data, i.e., $E[\rho_u; \rho_y^\delta] := D(\calG\#\rho_u, \rho_y^\delta)$. With this definition, we rewrite~\eqref{eqn:optimization_sto}:
\begin{equation}\label{eqn:inf_SIP}
\rho_u^\delta=\argmin_{\rho_u \in \calP(\calD)}\; E[\rho_u;\rho^\delta_y] := \argmin_{\rho_u \in \calP(\calD)}\; D(\mathcal{G}\#\rho_u , \rho^\delta_y) \,,
\end{equation}
where $\calP(\calD)$ is the feasible set. The set may not necessarily be metricized. Here, $D$ can be any user-chosen distance or divergence between two probability measures. The given data $\rho_y^\delta$ is $\delta$-away from the ground truth $\rho^\ast_y=\calG\#\rho_u^\ast\in\calP(\calR)$ according to a certain metric/divergence. This objective functional is the most straightforward formulation derived from Problem I. We examine some theoretical foundations in Section~\ref{sec:vf_existence}, including the existence of the minimizer for the variational problem~\eqref{eqn:inf_SIP}.

The second formulation aims to address the ill-conditioning issue of the inversion. Just as in the deterministic setting where a regularization term is added to improve the conditioning of the problem, when the data given and the to-be-reconstructed objects are both probability measures, regularization also provides a mean to tame instability. In this setting,~\eqref{eqn:optimization_sto} changes to:
\begin{equation}\label{eqn:regularize}
\rho_u^\delta=\argmin_{\rho_u \in \calP(\calD)}\; E[\rho_u;\rho^\delta_y] := \argmin_{\rho_u \in \calP(\calD)}\;D(\mathcal{G}\#\rho_u, \rho_y^\delta ) + \mathsf{R}(\rho_u)\,,
\end{equation}
where $\mathsf{R}:\mathcal{P}(\mathcal{D}) \to [0, \infty)$ is a specifically designed regularizer. Depending on the structure of $\mathsf{R}$, different properties are enhanced. We study various regularizers in Section~\ref{sec:regularize}.

\subsection{Existence of the solution to the variational framework}\label{sec:vf_existence}
First, we study the variational framework in its most straightforward formulation~\eqref{eqn:inf_SIP}, where the objective functional is the plain evaluation of the  distance $D$ between simulated data $\calG\#\rho_u$ and the reference data distribution $\rho_y^\delta$.

Even in this very simple setting, noting that the problem has an infinite dimensional feasible set, the existence may not be completely trivial. In general, a converging sequence can easily converge to a point outside the feasible set if the set is not compact. Certain conditions on the regularity of $E[\rho_u;\rho_y^\delta]$ and the closeness of the feasible set need to be specified. To this end, we cite the following general result on the existence of minimizers; see for instance \cite{ambrosio2005gradient, santambrogio2015optimal, calder2020calculus}.
\begin{theorem} \label{existence}
    We consider the topology induced by the weak convergence over the space of probability measures $\calP(X)$ where $X$ is a Polish space. If the functional $E: \mathcal P(X) \rightarrow [0, \infty)$ is
    \begin{itemize}
        \item lower semicontinuous (l.s.c.), i.e., for every $\rho_u^1 \in \mathcal P(X)$
    \[
    E(\rho_u^1) \leq \liminf_{\rho_u^2 \rightarrow \rho^1_u} E(\rho^2_u) \,, \text{~~where~} \rho_u^2 \rightarrow \rho^1_u \text{~ in the topology of~} \mathcal P(X)\,,
    \]
    \item coercive, i.e., for $\lambda > \inf_{\rho_u\in \mathcal P(X)} E(\rho_u)$, the set
    \[
    A = \{\rho_u \in \mathcal P(X): ~ E(\rho_u)< \lambda \}
    \]
    is sequentially precompact,
    \end{itemize}
    then there exists $\rho_u^* \in \mathcal P(X)$ such that $ E(\rho_u^*) = \min_{\rho_u\in \mathcal P(X)} E(\rho_u)\,.$
\end{theorem}
Theorem~\ref{existence} gives a quick corollary in our setting.
\begin{theorem}\label{thm:our_existence}
Let $\calD$ be a Polish space. For any fixed $\rho^\delta_y$, if
    \begin{itemize}
        \item $D(\,\cdot\,, \rho^\delta_y): \mathcal P(\mathbb{R}^n)\rightarrow [0,\infty]$ is lower semicontinuous and coercive with respect to the topology chosen for $\calP(\mathbb{R}^n)$,
        \item $\mathcal G: \calD\rightarrow \calR := \calG(\calD)$ is open and continuous {where $\mathcal R\subset \mathbb R^n$ is closed},
    \end{itemize}
    then there exists a minimizer of \eqref{eqn:inf_SIP} in $\calP(\calD)$.
\end{theorem}
\begin{proof}
To see this, we first claim:
\begin{equation}\label{eq:two_inf_equal}
\inf_{\tilde\rho_y \in \calP(\calR)} D(\tilde\rho_y , \rho^\delta_y) = \inf_{\rho_u \in \calP(\calD )} D(\calG \# \rho_u , \rho^\delta_y)\,.
\end{equation}
This amounts to proving that
\begin{equation}\label{eq:two_sets_equal}
\{\calG \# \rho_u,\,\, \forall \rho_u \in \calP(\calD)\} = \calP(\calR)\,.
\end{equation}
The ``$\subseteq$'' direction is apparent, and to show ``$\supseteq$'', we note that for any $y \in \calR$, $\calG^{-1}(y) \neq \emptyset$. This allows us to define an equivalent relation $\sim$ on $\calD$: $u_1\sim u_2$ if $\calG(u_1) = \calG(u_2)$. We can then define the quotient set $\Omega := \calD \setminus \hspace{-1mm} \sim$. Consequently, $\calG:\Omega \to \calR$ is a bijection with a well-defined inverse $\calG^{-1}$. For any $\rho_y \in \calP(\calR)$, we identify one distribution $\rho_u:= \calG^{-1} \# \rho_y \in \calP(\Omega)$ satisfying $\calG\# \rho_u  = \rho_y$. Therefore,
\[
\calP(\calR) \subseteq \{\calG \# \rho_u,\,\, \forall \rho_u \in \calP(\Omega)\}  \subseteq \{\calG \# \rho_u,\,\, \forall \rho_u \in \calP(\calD)\}\,.
\]
This proves the ``$\supseteq$'' direction of~\eqref{eq:two_sets_equal}. As a result, \eqref{eq:two_inf_equal} holds.

In the second step, we prove there exists a minimizer for
\[\inf_{\rho_y \in \calP(\calR)} D(\rho_y , \rho_y^\delta)\,.
\]
Since $\calD$ is Polish and $\calG:\calD \to \calR$ is open, continuous and onto, then $\calR$ is also Polish~\cite[Theorem 7.5]{hjorth2000classification}. Recall by assumption, $D(\cdot, \rho_y^\delta)$ as a functional over $\calP(\mathbb{R}^n)$ is l.s.c.~and coercive with respect to the weak convergence topology. When restricting the domain from $\calP(\bbR^d)$ to $\calP(\calR)$, $D(\,\cdot\, , \rho_y^\delta)$ still inherits these two properties. For the lower semi-continuity, consider any sequence $\{\rho_y^{n}\}\in\calP(\calR) \subseteq  \calP(\bbR^d)$ with weak limit   $\rho_y^{n} \rightarrow \widetilde{\rho_y}$ as $n \to \infty$. Note that $ \widetilde{\rho_y}\in \calP(\calR) $ due to the closedness of $\calP(\calR)$ under weak topology.  Since $D(\,\cdot\, , \rho_y^\delta)$ is l.s.c.~over $\calP(\bbR^d)$, we have
\[
E(\widetilde{\rho_y} ) \leq \liminf_{\rho_y^n \to \widetilde{\rho_y} } E(\rho_y^n) \,,
\]
which implies that $D(\,\cdot\, , \rho_y^\delta)$ is l.s.c.~over $\calP(\calR)$. Coercivity holds because a subset of a sequentially precompact set is still sequentially precompact. Therefore, by~Theorem~\ref{existence}, $D(\,\cdot\, , \rho_y^\delta)$ has a minimizer in $\calP(\calR)$, and by~\eqref{eq:two_inf_equal} and~\eqref{eq:two_sets_equal}, this corresponds to a minimizer $\rho_u \in \mathcal{P}(\calD)$  to~\eqref{eqn:inf_SIP}.
\end{proof}

\begin{remark}
Many common choices of divergences/metrics $D$ satisfy the conditions in Theorem~\ref{thm:our_existence}. For example, if $D$ is the $p$-Wasserstein metric, then the l.s.c.~of $E(\rho_u)$ follows from the l.s.c.~of the p-Wasserstein distance; see \cite[Corollary 6.11 and Remark 6.12]{villani2009optimal}. The coercivity follows from the fact that the finite ball in the $p$-Wasserstein metric is weakly compact~\cite[Theorem 1]{yue2022linear}. In the example of $\KL$-divergence, the l.s.c.~and the coercivity~follow from~\cite[Theorem 19-20]{van2014renyi}.
\end{remark}

\subsection{Variational formulation with regularization}\label{sec:regularize}
We now turn our attention to the regularized problem~\eqref{eqn:regularize}, where the regularizer $\mathsf{R}$ is added to promote certain properties of the reconstructed solution $\rho_u^\delta$ while taming the instability in the reconstruction.

Just as in the deterministic setting where different pairs of $(D,\sfR)$ enhance different properties of the reconstructed solution, we expect different designs of $\sfR$, when paired with various of $D$, to promote special properties of $\rho_u^\delta$ as well. Considering all such possible pairings is a vastly diverse topic. Here we confine ourselves to two cases:
\begin{itemize}
    \item Entropy-Entropy pair: we assume $D$ and $\sfR$ take on the form of relative entropy;
    \item $\wass_2$-$\wass_2$ pair: we assume both $D$ and $\sfR$ take the form of the Wasserstein distance.
\end{itemize}
We leave the examination of other possible $(D,\sfR)$ pairs to future work.

\smallskip

\noindent{\bf Case 1: Entropy-Entropy pair.} Set $D= \KL$ and $
\sfR(\rho_u)=\KL(\rho_u||\mathcal{M})$, with $\calM \in \calP(\calD)$ being a desired output measure for which $\frac{\rd \rho_u}{\rd \calM}$ exists. For the rest of this analysis we assume that all probability distributions are absolutely continuous with respect the the Lebesque measure on the corresponding spaces, and we use the same notation to refer to the distribution and its corresponding density interchangeably. Then~\eqref{eqn:regularize} becomes:
\begin{equation}\label{eq:regularized_KL}
\rho_u^\delta=\argmin_{\rho_u \in \mathcal P_{2,\text{ac}}}\; \KL(\mathcal{G}\#\rho_u|| \rho_y^\delta ) + \alpha  \int  \log \frac{\rho_u}{
\mathcal M}\, \rho_u \rd u =: \mathcal L(\rho_u)\,.
\end{equation}
Under these assumptions we have the following theorem.
\begin{theorem} \label{reg-KL}
Assume $\calG$ is invertible.    The optimal solution to \eqref{eq:regularized_KL} is
\begin{equation} \label{opt-inv}
\rho_u^\delta \propto [(\mathcal G^{-1}\# \rho_y^\delta) \mathcal M^\alpha ]^{\frac{1}{1+\alpha}}  \,.
\end{equation}
Let  $\rho_u^*=\calG^{-1}\# \rho_y^*$ be the ground truth. Then we have the following error estimate:
\begin{equation*}
    \KL(\rho_u^* || \rho_u^\delta) = \frac{1}{1+\alpha} \KL(\rho_y^* || \rho_y^\delta) + \frac{\alpha}{1+\alpha} \KL(\rho_y^* || \mathcal G\# \mathcal M) - \log C\,,
\end{equation*}
where $C$ is
\begin{equation} \label{C-norm}
    C = \left( \int [(\mathcal G^{-1}\# \rho_y^\delta) \mathcal M^\alpha ]^{\frac{1}{1+\alpha}}  \rd u\right)^{-1}\xrightarrow[]{\alpha \rightarrow 0} 1\,.
\end{equation}
\end{theorem}
\begin{proof}
Since the KL divergence is convex (in the usual sense) and the pushforward action is a linear operator, the optimal solution of \eqref{opt-inv} can be obtained by solving the optimality condition:
\begin{align*}
   C_0 = \frac{\delta \mathcal L}{\delta \rho_u} \big|_{\rho_u = \rho_u^\delta}
   & = 1+ \log \frac{\rho_u^\delta}{\calG^{-1} \# \rho_y^\delta} + \alpha \left[ 1+ \log \frac{\rho_u^\delta }{\calM }\right]\,,
\end{align*}
where $C_0$ is any constant and we have used the fact that $$\KL(\calG \# \rho_u || \rho_y^\delta) = \KL( \rho_u || \calG^{-1} \#\rho_y^\delta)$$. Clearly,
\[
\rho_u^\delta = C [(\mathcal G^{-1}\# \rho_y^\delta) \mathcal M^\alpha ]^{\frac{1}{1+\alpha}}\,,
\]
where $C$ is the normalizing constant \eqref{C-norm}.

Substituting \eqref{opt-inv} into $ \KL (\rho_u^*||\rho_u^\delta)$, we have
\begin{align*}
    \KL (\rho_u^*||\rho_u^\delta) & = \int \rho_u^*(u) \log \frac{\rho_u^*(u)}{\rho_u^\delta(u)}\, \rd u
    \\ & = \int \rho_u^*(u) \left\{ \log \rho_u^*(u) - \frac{1}{1+\alpha} \log [(\calG^{-1} \# \rho_y^\delta )(u)\, \calM(u)^\alpha ]   - \log C\right\} \, \rd u
    \\
    & = \frac{1}{1+\alpha} \int \rho_u^* \log \frac{\rho_u^*}{\calG^{-1} \# \rho_y^\delta} \, \rd u + \frac{\alpha}{1+\alpha} \int \rho_u^* (u)\log \frac{\rho_u^*(u)}{\calM (u)}\, \rd u - \log C
    \\ &= \frac{1}{1+\alpha} \KL(\rho_y^* || \rho_y^\delta) + \frac{\alpha}{1+\alpha} \KL\left(\rho_y^* || \mathcal G\# \mathcal M \right)- \log C \,.
\end{align*}
\end{proof}

\noindent{\bf Case 2: $\wass_2$-$\wass_2$ pair.} Here, we set $\sfR[\rho_u]= \int |u|^2 \rd \rho_u(u)$, the second-order moment of $\rho_u$, and $D = \wass_2$. Then~\eqref{eqn:regularize} becomes:
\begin{equation}\label{eq:regularized2}
\rho_u^\delta=\argmin_{\rho_u \in \mathcal P_2}\; \wass_2^2(\mathcal{G} \# \rho_u, \rho_y^\delta  ) + \alpha^2  \int |u|^2 \rd \rho_u(u)=:E[\rho_u;\rho_y^\delta]\,.
\end{equation}
One nice observation about this regularization is that
\[
\sfR[\rho_u]=\wass_2^2(\rho_u,\delta_0)\,,
\]
and therefore the  whole objective functional can be condensed into one, as shown in the lemma below.
\begin{lemma}\label{lem:was_reg_1}
For any $\rho^\delta_y \in \calP(\bbR^n)$, the cost function defined in~\eqref{eq:regularized2} can be rewritten as:
\begin{align} \label{W2-reg}
 E[\rho_u;\rho_y^\delta]=\wass_2^2(\mathcal G\# \rho_u, \rho^\delta_y) + \alpha^2 \int |u|^2 \rd \rho_u(u) = \wass_2^2(\tilde \calG\# \rho_u, \bar \rho_y) \,,
\end{align}
with $\bar \rho_y = \rho^\delta_y \otimes \delta_0(y)$ where $\delta_0(y) \in \calP(\mathbb{R}^{n})$ denotes the Dirac delta centered at $0\in \mathbb{R}^{n}$, and $\tilde{\mathcal G} =      \mathcal G  \otimes \alpha \Imat_m$, with $\Imat_m$ being the $m$-dimensional identity. More explicitly ,
\[
\tilde{\calG}(u):\calD\subset\mathbb{R}^m\to\calR\otimes\calD\subset\mathbb{R}^{n+m}\,,\quad\text{with}\quad \tilde{\calG}(u)=(\calG(u),\alpha u)\,.
\]
\end{lemma}
\begin{proof}
We drop sub-index $m$ in the proof because there is no ambiguity. Let $\pi_1$ be the optimal transport plan between $\calG\# \rho_u$ and $\rho^\delta_y$. Then
\[
\wass_2^2(\mathcal G\# \rho_u, \rho^\delta_y)  = \int |y' - y|^2 \pi_1 (\rd y' \rd y) = \int |\mathcal G(u) - y|^2 \hat{\pi}_1 (\rd u \rd y),
\]
where ${\pi}_1  = (\calG \times \Imat)\#  \hat {\pi}_1$ for some $\hat{\pi}_1 \in  \Gamma(\rho_u, \rho^\delta_y)$. Note that if $\calG$ is not one-to-one, $\hat{\pi}_1$ may not be unique, but its existence is always guaranteed. Similarly:
\begin{equation}
    \int |u|^2 \rd\rho_u=\int |u-0|^2 \hat{\pi}_2(\rd u \rd u')\,,\quad\text{with}\quad \hat{\pi}_2 = \rho_u \otimes \delta_0(u) \in \Gamma(\rho_u, \delta_0(u))\,,
\end{equation}
where $\delta_0(u) \in \calP(\mathbb{R}^m)$ denotes the Dirac delta at 0. Defining $\hat{\pi}_3  =\hat{\pi}_1\otimes \delta_0(u) \in \Gamma( \rho_u, \rho_y^\delta \otimes \delta_0(u) )$, we rewrite:
\begin{align*}
E[\rho_u;\rho_y^\delta] =  & \int |\mathcal G(u) - y|^2 \hat{\pi}_1 (\rd u \rd y) + \alpha^2 \int |u|^2 \rd\rho_u
    \\  = & \int |\tilde{\mathcal G} (u) - {\bf y}'|^2 \hat{\pi}_3(\rd u\, \rd {\bf y}')\quad\text{with} \quad {\bf y}'=(y,0)\\
    = &   \int |{\bf y} -{\bf y}'|^2 {\pi}_3 (\rd {\bf y}  \, \rd{\bf y}' ) \,, \quad {\pi}_3 = (\tilde{\mathcal{G}} \times \Imat)\#  \hat {\pi}_3  \in \Gamma\left( \tilde{\mathcal{G}}\# \rho_u, \, \rho_y^\delta \otimes \delta_0(u) \right)\,.
    \end{align*}
To show this is $\wass^2_2(\tilde{\calG}\#\rho_u,\bar\rho_y)$, we also need to show $\pi_3$ is an optimal plan. Assume $\gamma\neq \pi_3$ and $\gamma$ is the optimal transport plan between $ \tilde{\mathcal{G}}\# \rho_u$ and  $\bar{\rho}_y = \rho_y^\delta\otimes \delta_0(u)$, then we have
\begin{align*}
    \wass_2^2({\tilde \calG}\# \rho_u, \bar \rho_y) = &  \int |{\bf y} -{\bf y}'|^2  \rd\gamma (\rd {\bf y}  \, \rd{\bf y}' ) \\
    = & \int |\tilde{\mathcal G} (u) - {\bf y}'|^2  \rd\hat{\gamma}(\rd u\, \rd {\bf y}'),\quad \gamma = (\tilde{\mathcal{G}} \times \Imat)\# \hat{\gamma},\,\, \hat{\gamma}\in \Gamma\left(\rho_u,
       \rho_y^\delta \otimes \delta_0(u) \right) \\
    = & \int |\mathcal G(u) - y|^2  \rd \hat{\gamma}_1 (\rd u \,  \rd y) + \alpha^2 \int |u|^2  \rd \rho_u  \,, \quad \hat{\gamma}_1 \in \Gamma(\rho_u,
        \rho_y^\delta) \\
    = & \int |y - y'|^2  \rd \hat{\gamma}_2 (\rd y \, \rd y') + \alpha^2 \int |u|^2 \rd \rho_u  \,, \quad  \hat{\gamma}_2 \in \Gamma(\calG\# \rho_u,
        \rho_y^\delta) \\
    \geq & \,\, \wass_2^2(\mathcal G\# \rho_u, \rho_y^\delta) + \alpha^2 \int |u|^2 \rd \rho_u\,, \\
    = & \int |{\bf y} -{\bf y}'|^2 {\pi}_3 (\rd {\bf y}  \, \rd{\bf y}' ) \,,
\end{align*}
where $ \hat{\gamma}_1$ and $\hat{\gamma}_2$ are determined by $\gamma$. This contradicts the assumption that $\pi_3$ is not optimal. So we conclude with~\eqref{W2-reg}.
\end{proof}

This lemma holds for generic $\calG$. When $\calG$ is linear, the newly introduced regularizer brings effects that resonate Tikhonov regularization, as stated in the following theorem.

\begin{theorem} \label{reg-W}
Let $\mathcal G = \Amat\in\bbR^{n\times m}$ with $n\geq m$, $\Amat$ has full column rank, and $\Amat^\dagger=(\Amat^\top\Amat)^{-1}\Amat^\top$ as defined in~\eqref{eqn:A_penrose}. Then:
\begin{itemize}
    \item When $\delta=0$, $\alpha=0$ and $\rho_y^*\in\calP_{ac}(\mathbb{R}^n)$ , the minimizer to~\eqref{eq:regularized2} is:
    \begin{equation}\label{eqn:opt_delta_0}
    \rho_u^\ast=\Amat^\dagger\#\rho^*_y  \,,
\end{equation}
\item When $\delta\neq 0$, $\alpha\neq 0$ and $\rho_y^\delta\in\calP_{ac}(\mathbb{R}^n)$, the variational problem~\eqref{eq:regularized2} achieves minimum at
\begin{align} \label{rhous}
\rho_u^\delta = (\Amat^\top \Amat + \alpha^2 \Imat)^{-1} \Amat^\top\# \rho_y^\delta \,.
\end{align}
\end{itemize}
The reconstruction error against the optimal solution is:
\begin{align} \label{12}
\wass_2(\rho_u^\delta, \rho_u^*) \leq  \norm{(\Amat^\top \Amat + \alpha^2 \Imat)^{-1} \Amat^\top} \wass_2(\rho^*_y, \rho_y^\delta)
+
\norm{(\Amat^\top \Amat + \alpha^2 \Imat)^{-1} \Amat^\top - \Amat^{\dagger}}_2 \sqrt{{\mathbb{E}_{\rho_y^*} \left[y^2\right]}}\,.
\end{align}
Furthermore, if $\sigma_m=\sigma_{\min}(\Amat)$ is the smallest singular value for $\Amat$, then \eqref{12} can be further simplified to
{
\begin{align}\label{eqn:12_simplify}
\wass_2(\rho_u^\delta, \rho_u^*)
& \leq \frac{1}{2\alpha} \wass_2(\rho^*_y, \rho_y^\delta) +  \frac{\alpha^2}{\sigma_m(\sigma_m^2+\alpha^2)} \sqrt{\mathbb{E}_{\rho^*_y} \left[|y|^2\right]} \nonumber
\\ & \leq \frac{1}{2\alpha} \wass_2(\rho^*_y, \rho_y^\delta) +  \frac{\alpha}{2\sigma_m^2} \sqrt{\mathbb{E}_{\rho^*_y} \left[|y|^2\right]} \,.
\end{align}
}
\end{theorem}
\begin{proof}
A proof of ~\eqref{eqn:opt_delta_0} was drawn in~\cite[Theorem 4.7]{li2023differential}. To show~\eqref{rhous}, we note that  when $\calG=\Amat$, according to~Lemma~\ref{lem:was_reg_1}, the problem~\eqref{eq:regularized2} is equivalent to:
\[
\min_{\rho_u \in \calP_2} \wass_2^2(\tilde \Amat\# \rho_u, \bar \rho_y^\delta) \,,
\]
where $\bar \rho_y^\delta = \rho_y^\delta \otimes \delta_0(u)$ and $\tilde \Amat= \begin{pmatrix} \Amat \\ \alpha \Imat
\end{pmatrix}$ is over-determined. Using~\cite[Theorem 4.7]{li2023differential} again:
\[
\rho_u^\delta =  \tilde\Amat^\dagger\# \bar \rho_y^\delta\,.
\]
The proof of~\eqref{rhous} is complete noticing $\tilde{\Amat}^\top \# \bar \rho_y^\delta = \Amat^\top \# \rho_y^\delta$.

To show~\eqref{12}, we leverage the classical analysis for Tikhonov regularization by introducing a third term:
\begin{equation}\label{eqn:middle_term}
\tilde \rho_u=(\Amat^\top \Amat + \alpha^2 \Imat)^{-1} \Amat^\top\# \rho_y^*\,.
\end{equation}
By the triangle inequality, we have
\begin{equation*}
\wass_2(\rho_u^\delta, \rho_u^*) \leq
\wass_2(\rho_u^\delta, \tilde \rho_u) + \wass_2(\tilde \rho_u, \rho_u^*)\,.
\end{equation*}
The first term can be estimated using the continuity of the map $(\Amat^\top \Amat + \alpha^2 \Imat)^{-1} \Amat^\top$ and comparing~\eqref{eqn:middle_term} with~\eqref{rhous} by citing~\cite[Theorem 3.2]{ernst2022wasserstein}. The second term is estimated using \cite[Theorem 3.1]{baptista2023approximation}:
    \begin{eqnarray*}
        \wass^2_2(\tilde \rho_u, \rho_u^*) &=& \wass^2_2((\Amat^\top \Amat + \alpha^2 \Imat)^{-1} \Amat^\top\# \rho^*_y, \Amat^{\dagger}\# \rho^*_y) \\
        &\leq & \int \left| \left(\Amat^\top \Amat + \alpha^2 \Imat\right)^{-1} \Amat^\top y -\Amat^{\dagger} y \right|^2 \rd \rho^*_y \\
        &\leq & C_{\rho^*_y}\left\|\left(\Amat^\top \Amat + \alpha^2 \Imat\right)^{-1} \Amat^\top -\Amat^{\dagger}\right\|^2_2 \,
    \end{eqnarray*}
where $C_{\rho^*_y} = {\mathbb{E}_{\rho_y^*} \left[|y|^2\right]}$ is the second moment of $\rho_y^*$. To go from~\eqref{12} to~\eqref{eqn:12_simplify}, one simply uses the singular value decomposition of $\Amat$.
\end{proof}
\begin{remark}
Note that the two terms in \eqref{eqn:12_simplify} resemble the two sources of errors: the former being the noise in the measurement, and the latter coming from the regularization. Equating these two contributions leads to the optimal choice of $\alpha$:
{
    \[
    \alpha = \frac{\sigma_m \sqrt{\wass_2(\rho^*_y, \rho_y^\delta)}}{\left( {\mathbb{E}_{\rho_y^*} \left[|y|^2\right]}\right)^{\frac 14}}=\sigma_m \sqrt{ \frac{\wass_2(\rho^*_y, \rho_y^\delta)}{\wass_2(\rho_y^*,\delta_0)}} \,.
    \]
  }
\end{remark}

\section{Problem III: gradient flow}\label{sec:problemIII}
While the existence of a minimizer for the variational problem~\eqref{eqn:inf_SIP}, as discussed in Section~\ref{sec:problemII}, is crucial, it provides limited practical insight into solving the problem. Therefore, in this section, we focus on Problem III and examine the gradient flow formulation~\eqref{eqn:gradient_sto} as a method for solving~\eqref{eqn:inf_SIP}. Specifically, we concentrate on Wasserstein gradient flows, investigating their convergence properties and the necessary conditions  for the energy $E$. In this context,~\eqref{eqn:gradient_sto} takes the form:
\begin{equation}\label{eqn:gradient_w2}
\partial_t \rho_u = \nabla \cdot \left(\rho_u \nabla \frac{\delta E}{\delta \rho_u}\right)\,, \quad \text{with} \quad E[\rho_u; \rho^\delta_y] := D(\mathcal{G}\#\rho_u, \rho^\delta_y)\,.
\end{equation}

Since gradient information is utilized, we must at least assume differentiability of $E$ on the feasible set. To avoid unnecessary complications, throughout this section, we work exclusively for $\rho_u \in \mathcal{P}_{\ac}(\calD)$. We further assume $\rho_y^\delta \in  \mathcal{P}_{\ac}(\bbR^n)$, and that $E$ is smooth.

\subsection{Characterizations of the equilibrium}\label{subsec:equilibrium}
In this subsection, we characterize some properties of the gradient flow equilibrium.

The form of the equilibrium is highly dependent on the choice of $D$. First, we  examine the gradient flow when $D$ in~\eqref{eqn:gradient_w2} is an $f$-divergence as defined in~\eqref{eq:f-divergence}, and $\mathcal{G}$ is a general nonlinear map. We then constrain our analysis to the setting where $\mathcal{G} = \Amat$ is linear.

When $D$ is an $f$-divergence with $f$  strongly convex, the gradient flow~\eqref{eqn:gradient_sto} becomes:
\begin{equation*}
\partial_t \rho_u = \nabla \cdot \left( \rho_u \nabla_u f'\left(\frac{\rho_y}{\rho_y^\delta}(\mathcal{G}(u)) \right) \right)\,, \quad \text{with} \quad \rho_y = \mathcal{G} \# \rho_u\,.
\end{equation*}
When $D$ is chosen as the KL divergence, we can further deduce, following~\cite{li2023differential}, the evolution equation for $\rho_y$:
\begin{equation}\label{eq:full_y_GF}
\partial_t \rho_y = \nabla_y \cdot \left(\rho_y \, B(y) \, \nabla_y \log\left(\frac{\rho_y}{\rho_y^\delta}\right)\right)\,, \quad y \in \mathcal{R}\,,
\end{equation}
where $B(y) = C(\mathcal{G}^{-1}(y))$ and $C(u) = \left.\nabla_u \mathcal{G}\right|_{u} \, \cdot\left.\nabla_u \mathcal{G}\right|^\top_{u}$.

It is standard practice to show that the optimizer is an equilibrium, meaning that the right hand side of~\eqref{eqn:gradient_w} vanishes at the optimizer. Consider the constrained optimization problem, $\min E(\rho_u)$ within the set $\{\rho_u: \int \rho_u \,\mathrm{d}u = 1\}$, and let $\lambda$ be the Lagrange multiplier. The Lagrangian is given by:
\[
L = E(\rho_u) + \lambda\left(\int \rho_u \,\mathrm{d}u  - 1\right)\,.
\]
The optimizer satisfies the first-order optimality condition for $L$, so by taking the derivative with respect to $\rho_u$, we obtain:
\[
f'\left(\frac{\rho_y^\opt}{\rho_y^\delta} \left(\mathcal{G}(u)\right)\right) + \lambda = 0 \quad \Longrightarrow \quad \nabla_u \, f'\left(\frac{\rho_y^\opt}{\rho_y^\delta} \left(\mathcal{G}(u)\right)\right) = 0\,,
\]
where we used $\frac{\delta E}{\delta \rho_u}(u) = \frac{\delta D}{\delta \rho_y} \circ \mathcal{G}(u)$ and denoted $\rho_y^\opt = \mathcal{G} \# \rho_u^\opt$.

{
For the $f$-divergence formulation, we specify a reference probability measure $\nu_{\calR}$ on the attainable range $\calR=\calG(\calD)$. We consider
\[
E_{\nu_{\calR}}(\rho_u)
=
D_f(\calG\#\rho_u\|\nu_{\calR})
=
\int_{\calR}
f\left(\frac{\rd(\calG\#\rho_u)}{\rd\nu_{\calR}}\right)
\,\rd\nu_{\calR},
\]
on the class of admissible probability measures satisfying $\calG\#\rho_u\ll\nu_{\calR}$. Here $f\in C^2((0,\infty))$ is strictly convex, $f(1)=0$, and its value at zero is understood through its lower semicontinuous extension.

Writing
\[
h=\frac{\rd(\calG\#\rho_u)}{\rd\nu_{\calR}},
\]
the first variation, whenever well defined, is
\[
\frac{\delta E_{\nu_{\calR}}}{\delta\rho_u}(u)
=
f'\bigl(h(\calG(u))\bigr).
\]
Consequently, the Wasserstein gradient flow takes the form
\begin{equation}\label{eqn:gradient_w}
\partial_t\rho_u
=
\nabla_u\cdot
\left(
\rho_u\nabla_u
f'\bigl(h(\calG(u))\bigr)
\right).
\end{equation}
If $\rho_y^\delta(\calR)=1$, one may take $\nu_{\calR}=\rho_y^\delta$. When the attainable range has zero probability under $\rho_y^\delta$, the reference measure on that range must instead be specified explicitly based on the data distribution $\rho_y^\delta$ as we will specify later in this section.

Stationarity alone does not establish global optimality. For a Wasserstein gradient flow, the dissipation controls the gradient of the first variation only where the parameter density is positive. The following proposition characterizes this constraint and gives a sufficient condition for an equilibrium to be a global minimizer.

\begin{proposition}\label{prop1}
Let $\calD\subseteq\bbR^m$ be open, let $\calG:\calD\to\calR=\calG(\calD)$ be continuously differentiable, and let $\nu_{\calR}$ be a probability measure on $\calR$. Consider the functional $E_{\nu_{\calR}}$ defined above.

Let $\rho_u^\eq=p^\eq(u)\,\rd u$ be a finite-energy equilibrium of~\eqref{eqn:gradient_w}, with $p^\eq\in C(\calD)$, and set
\[
\rho_y^\eq=\calG\#\rho_u^\eq,
\qquad
h^\eq=\frac{\rd\rho_y^\eq}{\rd\nu_{\calR}},
\qquad
\Omega=\{u\in\calD:p^\eq(u)>0\}.
\]
Then, on each connected component $\Omega_j$ of $\Omega$, there is a constant $C_j>0$ such that
\begin{equation}\label{eq}
h^\eq(\calG(u))=C_j,
\qquad u\in\Omega_j.
\end{equation}
The constants may differ between connected components. If, in addition, $\calD$ is connected and $p^\eq(u)>0$ for every $u\in\calD$, then
\[
\rho_y^\eq=\nu_{\calR},
\]
and $\rho_u^\eq$ is a global minimizer of $E_{\nu_{\calR}}$.
\end{proposition}

\begin{proof}
Testing the stationary equation against $\psi^\eq$ and using the assumed integration-by-parts justification gives
\[
0
=
-\int_{\Omega}
p^\eq(u)\,|\nabla_u\psi^\eq(u)|^2\,\rd u.
\]
Since $p^\eq$ is positive on $\Omega$ and $\nabla_u\psi^\eq$ is continuous there, we obtain
\[
\nabla_u\psi^\eq=0
\qquad\text{on }\Omega.
\]
Thus $\psi^\eq$ is constant on each connected component $\Omega_j$.

Strict convexity of the differentiable function $f$ implies that $f'$ is strictly increasing and hence injective. Therefore, constancy of $f'(h^\eq\circ\calG)$ implies
\[
h^\eq(\calG(u))=C_j
\qquad\text{on }\Omega_j.
\]

If $p^\eq>0$ throughout the connected domain $\calD$, there is a single constant $C$. Since $\calG(\calD)=\calR$, it follows that $h^\eq=C$ $\nu_{\calR}$-almost everywhere. Normalization then gives
\[
1
=
\int_{\calR}h^\eq\,\rd\nu_{\calR}
=
C.
\]
Hence $\rho_y^\eq=\nu_{\calR}$. Finally, Jensen's inequality gives, for every admissible $\rho_u$,
\[
E_{\nu_{\calR}}(\rho_u)
=
\int_{\calR}f(h)\,\rd\nu_{\calR}
\geq
f\left(\int_{\calR}h\,\rd\nu_{\calR}\right)
=
f(1)=0.
\]
The equilibrium attains this value and is therefore a global minimizer.
\end{proof}

We next compare the equilibrium distributions associated with the squared Wasserstein discrepancy and an intrinsic $f$-divergence in the linear overdetermined setting.

Let $\Amat\in\bbR^{n\times m}$ have full column rank, with $n>m$, and write
\[
L=\Col(\Amat),
\qquad
P_L=\Amat\Amat^\dagger.
\]
Here $P_L$ is the orthogonal projection onto $L$. Let $\mathrm{vol}_L$ denote the $m$-dimensional Euclidean volume measure on $L$. Explicitly, if $V\in\bbR^{n\times m}$ has orthonormal columns spanning $L$, then
\[
\int_L g(y)\,\rd\mathrm{vol}_L(y)
=
\int_{\bbR^m}g(Vz)\,\rd z.
\]

Suppose that $\rho_y^\delta(\rd y)=q(y)\,\rd y\in\calP_2(\bbR^n)$, where $q$ is a positive continuous density, and assume
\[
0<Z_L:=
\int_L q(y)\,\rd\mathrm{vol}_L(y)<\infty,
\qquad
\int_L |y|^2q(y)\,\rd\mathrm{vol}_L(y)<\infty.
\]
Define the normalized slice measure
\begin{equation}\label{eq:reference_on_L}
\nu_L(\rd y)
=
\frac{q(y)}{Z_L}\,\rd\mathrm{vol}_L(y).
\end{equation}

This definition specifies an intrinsic reference measure. Indeed, $\rho_y^\delta(L)=0$, whereas $(\Amat\#\rho_u)(L)=1$. Consequently,
\[
\KL(\Amat\#\rho_u\|\rho_y^\delta)=+\infty.
\]
The intrinsic divergence $D_f(\Amat\#\rho_u\|\nu_L)$ is therefore a revised objective.

\begin{theorem}\label{thm:equilibrium_form}
Under the assumptions above, consider
\[
E_W(\rho_u)
=
\wass_2^2(\Amat\#\rho_u,\rho_y^\delta),
\qquad
E_f(\rho_u)
=
D_f(\Amat\#\rho_u\|\nu_L).
\]
The minimization is over $\calP_2(\bbR^m)$, with $\Amat\#\rho_u\ll\nu_L$.

\begin{itemize}
\item
The functional $E_W$ has the unique minimizer
\[
\rho_{u,W}^\opt=\Amat^\dagger\#\rho_y^\delta,
\]
whose push-forward is
\[
\rho_{y,W}^\opt=P_L\#\rho_y^\delta.
\]
Thus the data equilibrium is the orthogonal projection, or marginal distribution, of $\rho_y^\delta$ onto $L$.

\item
The functional $E_f$ has the unique minimizer
\[
\rho_{u,f}^\opt=\Amat^\dagger\#\nu_L,
\]
whose push-forward is
\[
\rho_{y,f}^\opt=\nu_L.
\]
Thus the data equilibrium is the normalized slice measure~\eqref{eq:reference_on_L}.
\end{itemize}
\end{theorem}

\begin{proof}
Set $\eta=P_L\#\rho_y^\delta$. For every $\rho_y\in\calP_2(L)$, orthogonal decomposition of the quadratic transport cost yields
\begin{equation}\label{eq:projection_decomposition}
\wass_2^2(\rho_y,\rho_y^\delta)
=
\wass_2^2(\rho_y,\eta)
+
\int_{\bbR^n}|y-P_Ly|^2\,\rd\rho_y^\delta(y).
\end{equation}
To obtain equality, a coupling of $\rho_y$ and $\eta$ can be lifted to a coupling of $\rho_y$ and $\rho_y^\delta$ by disintegrating $\rho_y^\delta$ over the projection $P_L$.

The second term in~\eqref{eq:projection_decomposition} is independent of $\rho_y$. Hence the unique minimizing data measure is $\eta$. Since $\Amat:\bbR^m\to L$ is bijective, the corresponding parameter measure is uniquely determined:
\[
\Amat^\dagger\#\eta
=
\Amat^\dagger\#\rho_y^\delta.
\]

To characterize regular equilibria of $E_W$, let $\rho_u^\eq=p^\eq\,\rd u$ and $\rho_y^\eq=\Amat\#\rho_u^\eq$. Because $\rho_y^\eq$ is absolutely continuous with respect to $\mathrm{vol}_L$, there is a quadratic optimal transport map $T:L\to L$ from $\rho_y^\eq$ to $\eta$. The gradient of the first variation is
\[
\nabla_u\frac{\delta E_W}{\delta\rho_u}(u)
=
2\Amat^\top\bigl(\Amat u-T(\Amat u)\bigr)
\qquad
\rho_u^\eq\text{-almost everywhere}.
\]
Testing the stationary equation against the first variation therefore gives
\[
0
=
4\int_{\bbR^m}
\left|
\Amat^\top\bigl(\Amat u-T(\Amat u)\bigr)
\right|^2
\,\rd\rho_u^\eq(u).
\]
Both $\Amat u$ and $T(\Amat u)$ belong to $L$, and $\Amat^\top$ is injective on $L$. It follows that $T(y)=y$ for $\rho_y^\eq$-almost every $y$. Since $T\#\rho_y^\eq=\eta$, we conclude that $\rho_y^\eq=\eta$. Injectivity of $\Amat$ gives the claimed parameter law.

For $E_f$, Jensen's inequality and strict convexity imply
\[
D_f(\rho_y\|\nu_L)\geq0,
\]
with equality if and only if $\rho_y=\nu_L$. The measure $\Amat^\dagger\#\nu_L$ belongs to $\calP_2(\bbR^m)$ and has push-forward $\nu_L$. It is therefore the unique minimizer. The assertion concerning strictly positive equilibria follows from Proposition~\ref{prop1}, taking $\calD=\bbR^m$ and $\nu_{\calR}=\nu_L$.
\end{proof}

\begin{remark}
The measure $\nu_L$ is a normalized slice of the specified continuous density $q$. Since $\rho_y^\delta(L)=0$, it is not defined by ordinary conditioning on the event $Y\in L$.
\end{remark}
}

We highlight the difference between these two types of equilibrium distributions in Figure~\ref{fig:compare}. This contrast is alarming and suggests the use of caution in making the choice of objective functional when solving stochastic inverse problems.

\begin{figure}
    \centering
    \subfloat[$\wass_2$ gradient  flow of KL]{\includegraphics[width = 0.5\textwidth]{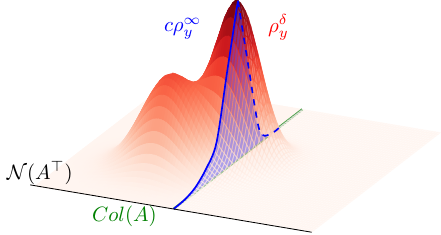}} \subfloat[$\wass_2$ gradient  flow of $\wass_2^2$]{\includegraphics[width = 0.5\textwidth]{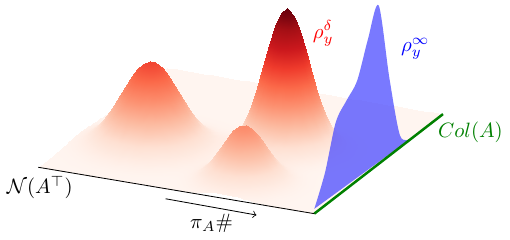}}
    \caption{In the over-determined case, the $\wass_2$ gradient flow of $\KL$ divergence and the squared $\wass_2$ metric between $\Amat \# \rho_u$ and $\rho_y^\delta$ have two different steady states $\rho_y^{\infty}$. The $\KL$ divergence recovers the \textit{conditional} distribution of $\rho_y^\delta$ on $\Col(\Amat)$ while the squared $\wass_2$ metric recovers the \textit{marginal} distribution of $\rho_y^\delta$ on $\Col(\Amat)$.\label{fig:compare}}
\end{figure}

{
\subsection{Exponential convergence}

The equilibrium characterization in Section~\ref{subsec:equilibrium} does not by itself establish convergence of the flow. We now obtain exponential decay of the relative entropy under strong log-concavity of the reference density. For linear maps, the analysis takes place on the attainable subspace. We subsequently give a nonlinear extension under a uniform lower bound on the Jacobian.

\begin{theorem}\label{thm:convergence}
Let $\Amat\in\bbR^{n\times m}$ have rank $r\geq1$, and set $L=\Col(\Amat)$. Let $\rho_y^\delta(\rd y)=q(y)\,\rd y$ be a probability measure, where $q>0$, $q\in C^2(\bbR^n)$, and
\[
-\nabla_y^2\log q(y)\succeq\lambda\Imat_n
\qquad\text{for every }y\in\bbR^n,
\]
with $\lambda>0$.

Define
\[
Z_L=\int_L q(y)\,\rd\mathrm{vol}_L(y),
\qquad
\nu_L(\rd y)=\frac{q(y)}{Z_L}\,\rd\mathrm{vol}_L(y),
\]
where $\mathrm{vol}_L$ is the $r$-dimensional Euclidean volume measure on $L$. Assume that $0<Z_L<\infty$ and $\nu_L\in\calP_2(L)$.

Let $\rho_u(t)$ be a global solution of the Wasserstein gradient flow of
\[
E_L(\rho_u)=\KL(\Amat\#\rho_u\|\nu_L),
\]
and write $\rho_y(t)=\Amat\#\rho_u(t)$. Assume sufficient regularity and decay to justify the entropy dissipation calculation below, and suppose $\KL(\rho_y(0)\|\nu_L)<\infty$. Then
\begin{equation}\label{eq:expo_contract}
\KL(\rho_y(t)\|\nu_L)
\leq
e^{-2\lambda\sigma_{\min}^2t}
\KL(\rho_y(0)\|\nu_L),
\end{equation}
where $\sigma_{\min}$ is the smallest positive singular value of $\Amat$.
\end{theorem}

\begin{proof}
Write
\[
p_t=\frac{\rd\rho_y(t)}{\rd\mathrm{vol}_L},
\qquad
q_L=\frac{q|_L}{Z_L}.
\]
Pushing the parameter-space gradient flow forward under $\Amat$ gives the intrinsic equation
\begin{equation}\label{eqn:FP_y}
\partial_t p_t
=
\operatorname{div}_L
\left(
p_t\Amat\Amat^\top
\nabla_L\log\frac{p_t}{q_L}
\right).
\end{equation}
Here $\nabla_L$ and $\operatorname{div}_L$ denote the Euclidean gradient and divergence on $L$.

Take the compact singular value decomposition
\[
\Amat=\Vmat\Sigma\Umat^\top,
\]
where $\Vmat\in\bbR^{n\times r}$ and $\Umat\in\bbR^{m\times r}$ have orthonormal columns, and $\Sigma$ contains the positive singular values. The map $z\mapsto\Vmat z$ is an isometry from $\bbR^r$ onto $L$.

Define the coordinate densities
\[
h_t(z)=p_t(\Vmat z),
\qquad
q_z(z)=q_L(\Vmat z)=\frac{q(\Vmat z)}{Z_L}.
\]
Equation~\eqref{eqn:FP_y} becomes
\begin{equation}\label{eqn:FP_y_3}
\partial_t h_t
=
\nabla_z\cdot
\left(
h_t\Sigma^2\nabla_z\log\frac{h_t}{q_z}
\right).
\end{equation}
Since the coordinate map is a measurable bijection, relative entropy is invariant under this change of coordinates:
\[
H(t):=\KL(\rho_y(t)\|\nu_L)
=
\int_{\bbR^r}
h_t(z)\log\frac{h_t(z)}{q_z(z)}\,\rd z.
\]

Mass conservation and integration by parts yield
\begin{equation}\label{eqn:rho_y_exp_1}
\begin{aligned}
H'(t)
=
-\int_{\bbR^r}
\left|
\Sigma\nabla_z\log\frac{h_t}{q_z}
\right|^2h_t\,\rd z \leq
-\sigma_{\min}^2
\int_{\bbR^r}
\left|
\nabla_z\log\frac{h_t}{q_z}
\right|^2h_t\,\rd z.
\end{aligned}
\end{equation}

Moreover,
\[
-\nabla_z^2\log q_z(z)
=
\Vmat^\top
\bigl[-\nabla_y^2\log q(\Vmat z)\bigr]\Vmat
\succeq\lambda\Imat_r.
\]
The Bakry--\'Emery criterion~\cite{Bakry1985} therefore gives the logarithmic Sobolev inequality
\begin{equation}\label{eqn:lsi}
H(t)
\leq
\frac{1}{2\lambda}
\int_{\bbR^r}
\left|
\nabla_z\log\frac{h_t}{q_z}
\right|^2h_t\,\rd z.
\end{equation}
Combining~\eqref{eqn:rho_y_exp_1} and~\eqref{eqn:lsi}, we obtain
\[
H'(t)\leq-2\lambda\sigma_{\min}^2H(t).
\]
Gr\"onwall's inequality proves~\eqref{eq:expo_contract}.
\end{proof}

\begin{remark}
We have a few remarks regarding the above theorem.
\begin{itemize}
\item
Strong log-concavity is a sufficient condition for the logarithmic Sobolev inequality. More generally, if the coordinate reference density satisfies
\[
\KL(h\,\rd z\|q_z\,\rd z)
\leq
\frac{1}{2\kappa}
\int_{\bbR^r}
\left|\nabla_z\log\frac{h}{q_z}\right|^2h\,\rd z
\]
for some $\kappa>0$, the same argument gives the decay rate $2\kappa\sigma_{\min}^2$.

\item
The theorem applies to every linear map of positive rank. If $\operatorname{rank}(\Amat)=n$, then $L=\bbR^n$, $Z_L=1$, and $\nu_L=\rho_y^\delta$. If $\operatorname{rank}(\Amat)<n$, convergence takes place on $L$ toward the normalized slice measure $\nu_L$. Although $\Amat\Amat^\top$ is singular on $\bbR^n$, its restriction to $L$ is positive definite.

\item
The estimate concerns the push-forward law $\rho_y(t)$. When $\Amat$ has a nontrivial kernel, it does not identify a unique limiting parameter law.
\end{itemize}
\end{remark}
For a nonlinear map, the same entropy argument applies under a uniform lower bound on the Jacobian. The induced data-space mobility generally depends on the current parameter distribution.

More precisely, let $\rho_{u,t}^{\,y}$ denote a regular conditional law of $u$ given $\calG(u)=y$ under $\rho_u(t)$. Whenever the required integrability holds, the mobility is
\begin{equation}\label{eq:nonlinear_mobility}
B_t(y)
=
\int_{\calG^{-1}(\{y\})}
D\calG(u)D\calG(u)^\top\,\rd\rho_{u,t}^{\,y}(u).
\end{equation}
Thus, writing $\rho_y(t,\rd y)=r_t(y)\,\rd y$, the push-forward equation is
\[
\partial_t r_t
=
\nabla_y\cdot
\left(
r_t B_t(y)\nabla_y\log\frac{r_t}{q}
\right).
\]

\begin{corollary}\label{cor:convergence}
Let $\calG:\bbR^m\to\bbR^n$ be a $C^2$ surjection. Assume that there exists $\sigma_0>0$ such that
\[
D\calG(u)D\calG(u)^\top
\succeq\sigma_0^2\Imat_n
\qquad\text{for every }u\in\bbR^m.
\]
Let $\rho_y^\delta(\rd y)=q(y)\,\rd y$, where $q>0$, $q\in C^2(\bbR^n)$, and
\[
-\nabla_y^2\log q(y)\succeq\lambda\Imat_n
\qquad\text{for every }y\in\bbR^n,
\]
with $\lambda>0$. Let $\rho_u(t)$ be a global Wasserstein gradient flow of
\[
E(\rho_u)=\KL(\calG\#\rho_u\|\rho_y^\delta).
\]
Assume sufficient regularity and integrability to justify the chain rule and entropy dissipation below, and suppose $\KL(\calG\#\rho_u(0)\|\rho_y^\delta)<\infty$. Then, with $\rho_y(t)=\calG\#\rho_u(t)$,
\begin{equation}\label{eq:expo_contract_2}
\KL(\rho_y(t)\|\rho_y^\delta)
\leq
e^{-2\lambda\sigma_0^2t}
\KL(\rho_y(0)\|\rho_y^\delta).
\end{equation}
\end{corollary}

\begin{proof}
Write $\rho_y(t,\rd y)=r_t(y)\,\rd y$ and set
\[
g_t(y)=\nabla_y\log\frac{r_t(y)}{q(y)},
\qquad
H(t)=\KL(\rho_y(t)\|\rho_y^\delta).
\]
The first variation is
\[
\frac{\delta E}{\delta\rho_u}(u)
=
1+\log\frac{r_t(\calG(u))}{q(\calG(u))}.
\]
Consequently, entropy dissipation along the parameter-space Wasserstein gradient flow gives
\[
\begin{aligned}
H'(t)
&=
-\int_{\bbR^m}
\left|
D\calG(u)^\top g_t(\calG(u))
\right|^2\,\rd\rho_u(t)(u)\\
&\leq
-\sigma_0^2
\int_{\bbR^m}
|g_t(\calG(u))|^2\,\rd\rho_u(t)(u)\\
&=
-\sigma_0^2
\int_{\bbR^n}|g_t(y)|^2r_t(y)\,\rd y.
\end{aligned}
\]
The logarithmic Sobolev inequality for $q$ yields
\[
H'(t)\leq-2\lambda\sigma_0^2H(t).
\]
The conclusion follows from Gr\"onwall's inequality.
\end{proof}
}

\section*{Acknowledgement}
Q.~Li is partially supported by NSF grants DMS-2308440, DMS-2023239, and DMS-2608503. L.~Wang is partially supported by NSF grants DMS-1846854 and DMS-2513336 and by UMN DSI-SSG-4886888864. M.~Oprea and Y.~Yang acknowledge support from NSF through grant DMS-2409855, Office of Naval Research through grant N00014-24-1-2088, and Cornell PCCW Affinito-Stewart Grant. Y.~Yang is additionally supported by NSF grant DMS-2540324.

\bibliographystyle{siam}
\bibliography{references}

\end{document}

%% file: Figures/diagramSIP.tex
\begin{tikzpicture}[scale=1.0]
    \path[<->, thick] (-1, -0.1) edge [bend right] node[right, xshift=5mm, yshift=0mm] {$u \sim \rho_u$}(-0.7, -2);
    \path[<->,thick] (4.2, -0.1) edge [bend left] node[left, xshift=-3mm] {$y \sim \rho_y$} (4.5, -2);
    \draw[smooth cycle, pattern color=orange, pattern=crosshatch dots]
        plot coordinates {(-1.5,0) (-2, 1.2) (-0.5,2) (1, 0.4)}
        node [label={[label distance=0.3cm, xshift=-1.4cm, yshift = -0.5cm, fill=white]:$\rho_{u} \in \mathcal{P}(\mathcal{D})$}] {};
    \draw[smooth cycle, pattern color=blue, pattern=crosshatch dots]
        plot coordinates {(4,0) (2.5, 2.2) (4.5,1.7) (6.6, 1)}
        node [label={[label distance=-0.8cm, xshift=-1.7cm, yshift=0.3cm, fill=white]:$\rho_y\in \mathcal{P}(\mathcal{R})$}] {};
    \draw[smooth cycle, pattern color=orange, pattern=north east lines] plot coordinates {(-1.5, -3.5) (-1.5, -2.2) (0.2, -2.2) (0.2, -3.5)}
        node [label={[label distance=0cm, xshift=-0.8cm, yshift = 0.2cm, fill=white]:$u \in \mathcal{D}$}] {};
    \draw[smooth cycle, pattern color=blue, pattern=north east lines]
        plot coordinates {(3.5, -3.5) (3.5, -2.2) (5.2, -2.2) (5.2, -3.5)}
        node [label={[label distance=-0cm, xshift=-0.8cm, yshift=0.2cm, fill=white]:$y\in \mathcal{R}$}] {};
    \draw[thick, ->] (0.8,1) -- (2.8, 1) node [label={[xshift=-1cm, above]:$\rho_y = \mathcal{G}_\sharp \rho_u$}] {};
    \draw[thick, ->] (0.8,-3) -- (2.8, -3) node [label={[xshift=-1cm, above]:$y = \mathcal{G}(u)$}] {};
\end{tikzpicture}

%% file: Figures/determinisitc.tex
\tikzset{every picture/.style={line width=0.75pt}}
\begin{tikzpicture}[x=0.75pt,y=0.75pt,yscale=-1,xscale=1]
\draw  [color=orange, draw opacity=1 ][fill=orange ,fill opacity=0.43 ] (169.93,85.89) -- (361.65,38.9) -- (302.07,151.21) -- (110.35,198.2) -- cycle ;
\draw [color={rgb, 255:red, 74; green, 144; blue, 226 }  ,draw opacity=1 ][line width=1.5]    (236,135.1) -- (190.98,4.94) ;
\draw [shift={(190,2.1)}, rotate = 70.92] [color={rgb, 255:red, 74; green, 144; blue, 226 }  ,draw opacity=1 ][line width=1.5]    (14.21,-4.28) .. controls (9.04,-1.82) and (4.3,-0.39) .. (0,0) .. controls (4.3,0.39) and (9.04,1.82) .. (14.21,4.28)   ;
\draw  [fill={rgb, 255:red, 0; green, 0; blue, 0 }  ,fill opacity=1 ] (171.83,146.39) .. controls (171.75,148.27) and (170.16,149.74) .. (168.27,149.66) .. controls (166.39,149.58) and (164.92,147.99) .. (165,146.1) .. controls (165.08,144.21) and (166.67,142.75) .. (168.56,142.83) .. controls (170.44,142.91) and (171.91,144.5) .. (171.83,146.39) -- cycle ;
\draw [color={rgb, 255:red, 128; green, 128; blue, 128 }  ,draw opacity=1 ]   (125,118.1) -- (139.41,165.19) ;
\draw [shift={(140,167.1)}, rotate = 252.98] [color={rgb, 255:red, 128; green, 128; blue, 128 }  ,draw opacity=1 ][line width=0.75]    (10.93,-3.29) .. controls (6.95,-1.4) and (3.31,-0.3) .. (0,0) .. controls (3.31,0.3) and (6.95,1.4) .. (10.93,3.29)   ;
\draw [color={rgb, 255:red, 128; green, 128; blue, 128 }  ,draw opacity=1 ]   (148,45.1) -- (189.09,32.68) ;
\draw [shift={(191,32.1)}, rotate = 163.18] [color={rgb, 255:red, 128; green, 128; blue, 128 }  ,draw opacity=1 ][line width=0.75]    (10.93,-3.29) .. controls (6.95,-1.4) and (3.31,-0.3) .. (0,0) .. controls (3.31,0.3) and (6.95,1.4) .. (10.93,3.29)   ;
\draw    (124,23.1) -- (168.27,149.66) ;
\draw    (179,182.1) -- (193,223.1) ;
\draw  [dash pattern={on 4.5pt off 4.5pt}]  (168.27,149.66) -- (179,182.1) ;
\draw (339.99,101.42) node [anchor=north west][inner sep=0.75pt]  [rotate=-359.95]  {$\textcolor{orange}{\mathcal{N}( A)^{\perp }= \text{Row}(A)} $};
\draw (221,26.4) node [anchor=north west][inner sep=0.75pt]    {$\textcolor[rgb]{0.29,0.56,0.89}{\mathcal{N}( A)}$};
\draw (175,136.4) node [anchor=north west][inner sep=0.75pt]    {$A^{\dagger } \yvec$};
\draw (105,16.4) node [anchor=north west][inner sep=0.75pt]    {$S_\yvec$};
\end{tikzpicture}